\documentclass[sn-basic,iicol]{sn-jnl}
\usepackage{graphicx}%
\usepackage{multirow}%
\usepackage{amsmath,amssymb,amsfonts}%
\usepackage{amsthm}%
\usepackage{mathrsfs}%
\usepackage[title]{appendix}%
\usepackage{xcolor}%
\usepackage{textcomp}%
\usepackage{manyfoot}%
\usepackage{booktabs}%
\usepackage{algorithm}%
\usepackage{algorithmicx}%
\usepackage{algpseudocode}%
\usepackage{listings}%

\newcommand{\ie}{i.e.}

\newcommand{\MC}{\mathcal}
\newcommand{\MBB}{\mathbb}

\usepackage{bm,subfigure}

\theoremstyle{thmstyleone}%
\newtheorem{theorem}{Theorem}
\theoremstyle{thmstyletwo}%
\newtheorem{remark}{Remark}%
\newtheorem{lemma}{Lemma}%
\newtheorem{corollary}{Corollary}%
\theoremstyle{thmstylethree}%
\newtheorem{definition}{Definition}%
\usepackage{makecell}
\begin{document}

\title[Article Title]{Towards Unsupervised Domain Adaptation via Domain-Transformer}

\author*[1]{\fnm{Chuan-Xian} \sur{Ren}}\email{rchuanx@mail.sysu.edu.cn}

\author[1]{\fnm{Yiming} \sur{Zhai}}\email{zhaiym3@mail2.sysu.edu.cn}
\author[1]{\fnm{You-Wei} \sur{Luo}}\email{luoyw28@mail2.sysu.edu.cn}
\author[2]{\fnm{Hong} \sur{Yan}}\email{h.yan@cityu.edu.hk}

\affil[1]{\orgdiv{School of Mathematics}, \orgname{Sun Yat-Sen University}, \orgaddress{\street{XingangRoad}, \city{Guangzhou}, \postcode{510275}, \state{Guangdong}, \country{China}}}
\affil[2]{\orgdiv{Department of Electrical and Engineering}, \orgname{City University of Hong Kong}, \orgaddress{\street{83 Tat Chee Avenue}, \city{Kowloon}, \state{Hong Kong}}}

\abstract{As a vital problem in pattern analysis and machine intelligence, Unsupervised Domain Adaptation (UDA) attempts to transfer an effective feature learner from a labeled source domain to an unlabeled target domain. Inspired by the success of the Transformer, several advances in UDA are achieved by adopting pure transformers as network architectures, but such a simple application can only capture patch-level information and lacks interpretability. To address these issues, we propose the Domain-Transformer (DoT) with domain-level attention mechanism to capture the long-range correspondence between the cross-domain samples. On the theoretical side, we provide a mathematical understanding of DoT: 1) We connect the domain-level attention with optimal transport theory, which provides interpretability from Wasserstein geometry; 2) From the perspective of learning theory, Wasserstein distance-based generalization bounds are derived, which explains the effectiveness of DoT for knowledge transfer. On the methodological side, DoT integrates the domain-level attention and manifold structure regularization, which characterize the sample-level information and locality consistency for cross-domain cluster structures. Besides, the domain-level attention mechanism can be used as a plug-and-play module, so DoT can be implemented under different neural network architectures. Instead of explicitly modeling the distribution discrepancy at domain-level or class-level, DoT learns transferable features under the guidance of long-range correspondence, so it is free of pseudo-labels and explicit domain discrepancy optimization. Extensive experiment results on several benchmark datasets validate the effectiveness of DoT.}

\maketitle

\section{Introduction}

Deep learning methods have achieved remarkable progress on various computer vision tasks, such as image classification~\citep{DR_CNN}, image segmentation~\citep{Long_2015_CVPR}, and object recognition~\citep{he2016deep}. It is worth noting that deep learning methods are highly dependent on large-scale labeled datasets and identical distribution assumption~\citep{long2019transferable}. However, data collected from the real-world scenarios are usually unlabeled, and manual annotation is always expensive and time-consuming. Moreover, there exists a domain shift since the training set and the test set may be collected with different visual conditions, such as sensors, backgrounds and views~\citep{pan2010domain,courty2016optimal}. To deal with the shortage of labeled data, a natural idea is transferring the knowledge from a label-rich domain (i.e., source domain) to a label-scarce domain (i.e., target domain). This is usually called Unsupervised Domain Adaptation (UDA) problem~\citep{pan2009survey,long2019transferable,yang2022heterogeneous}.

Generally, with the labeled source samples, UDA methods attempt to learn flexible feature representations that can generalize well to the unlabeled target domain. The UDA problem commonly assumes that the feature spaces of the source and target domains are the same but the marginal distributions are different, i.e., $P_X^s\neq P_X^t$. The rigorous transfer theory~\citep{ben2010theory} shows that the classifier's target error is mainly bounded by its source error plus the divergence between domains. Under the guidance of this theory, a major line of methods attempt to learn domain-invariant features by simultaneously minimizing the source error and statistical distribution discrepancy between domains, e.g., Maximum Mean Discrepancy (MMD)~\citep{pan2010domain,long2019transferable}, statistic moment matching~\citep{sun2016return}, manifold alignment~\citep{gong2012geodesic} and adversarial adaptation~\citep{ganin2016domain,ren2019domain}. Another fruitful line is based on Optimal Transport (OT) \citep{zhang2019optimal,li2020enhanced}, which matches the marginal distributions under a geometric analysis framework. With the marginal distribution constrains on the transport plan $\bm{\gamma}$, OT-based methods learn transferable features via a pair-wise matching between samples across domains.

\begin{figure}[t!]
\centering
\includegraphics[width=1\linewidth]{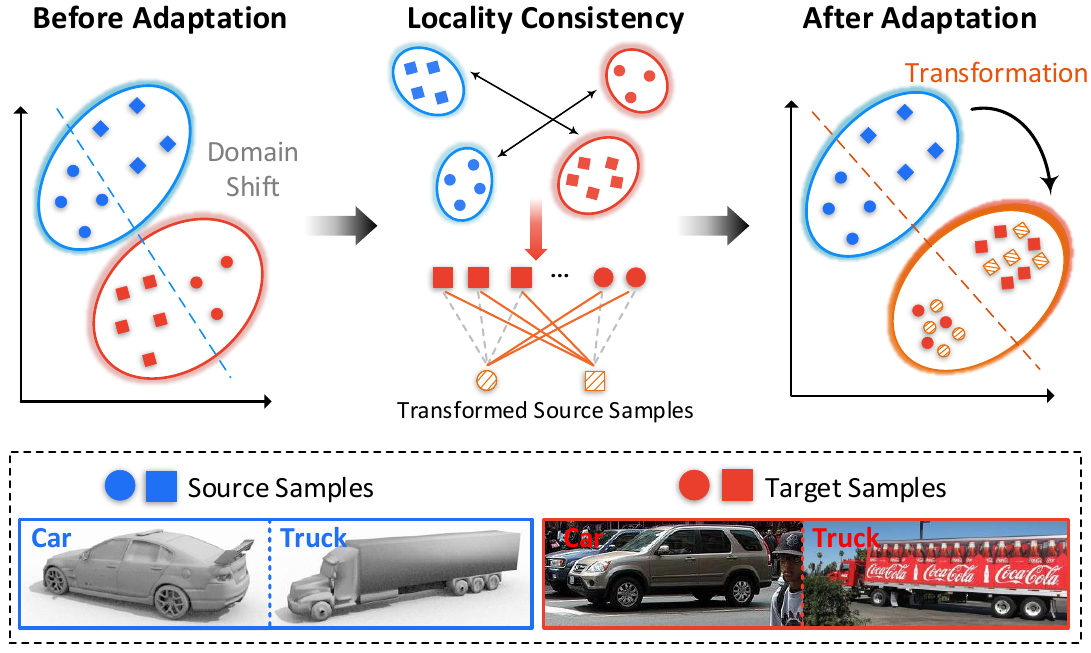}
\caption{To deal with UDA, our DoT learns locality consistency across domains, and then transforms the source domain into the target space to reduce the misclassification rate on the target domain. In the locality consistency part, black arrows imply alignment between class-relevant clusters across domains, orange solid lines imply large weights for intra-class samples, and gray dashed lines imply small and even zero weights for inter-class samples.}
\label{fig:UDA_OT}
\end{figure}

Extended from the marginal shift assumption, domain alignment with local structure consistency has been studied recently based on the conditional shift assumption~\citep{zhang2013}, i.e., $P^s_{X|Y} \neq P^t_{X|Y}$. As illustrated in the upper middle of Fig.~\ref{fig:UDA_OT}, locality consistency attempts to align class-relevant clusters and keep samples within the same class nearby in the latent feature space. \citet{Combes2020domain} have demonstrated the necessity of reducing conditional shift for UDA. Several works explore the locality consistency across domains by matching conditional and even joint distributions, e.g., extensions of MMD \citep{long-JAN,zhu2021deep}, conditional adversarial adaptation~\citep{long2018conditional}, centroid alignment \citep{yang2022heterogeneous}, OT-based joint distribution matching~\citep{jdot2017,Damodaran_2018_ECCV}, and conditional statistic alignment~\citep{BuresNet2023}. Most of these methods rely on target pseudo-labels since the target domain is unlabeled. Though false pseudo-labels are expected to be iteratively corrected in the training process, obtaining reliable pseudo-labels may be associated with several hyper-parameters and manually assigned threshold values. These phases introduce extra risk and complexity into the model training. Thus, aligning domains with locality consistency is still fundamental and challenging in dealing with UDA.

Transformer has achieved immense success in natural language processing (NLP)~\citep{vaswani2017attention} and computer vision~\citep{dosovitskiy2021an,DeiT}, which further inspires the recent advances in UDA~\citep{ma2021WinTR,yang2023tvt,xu2022cdtrans}. These methods focus on applying the self-attention~\citep{ma2021WinTR,yang2023tvt} or cross-attention~\citep{xu2022cdtrans} to learn transferrable representations. However, the simple application of vision transformers can only build interactions between image patches and cannot characterize relations between cross-domain samples. Thus, it is insufficient to learn the locality consistency for successful transfer. Besides, there lacks understanding on the effects of transformers for UDA, which leads to weak interpretability of these recent advances and limits further explorations of transformer-based UDA methodology.

In this paper, we propose a novel method called Domain-Transformer (DoT) for UDA, which consists of two core mechanisms, i.e., the domain-level attention and manifold regularization. It is worth noting that different from existing vision transformer-based UDA methods, which directly apply pure transformer framework and can only capture patch-level information, DoT characterizes the sample correspondence across domains by the proposed domain-level cross-attention mechanism. From theoretical aspects, we connect the core mechanism (i.e., domain-level attention) with OT theory, and then derive the Wasserstein distance-based generalization bounds; the main results provide a theoretical understanding and guarantees for the DoT model. Intuitively, domain-level attention builds intrinsic relationships across domains, and then transforms the source samples onto the target domain. In other words, the transformed source samples can be represented by the most similar target samples, as illustrated in Fig.~\ref{fig:UDA_OT}.

To the best of our knowledge, DoT is the first effort to connect an attention mechanism with the OT theory. The main contributions of this paper are summarized as follows:
\begin{itemize}
\item[\small\textbullet] Instead of explicitly modeling the domain discrepancy or directly applying pure transformer frameworks, DoT formulates the long-range dependency across domains by the proposed domain-level attention mechanism, which learns shared features and ensures locality consistency for knowledge transfer.

\item[\small\textbullet] The domain-level attention mechanism is theoretically connected with the OT theory. The long-range dependency across domains can be mathematically deduced as the transport map in OT. Thus, DoT has interpretability from the perspective of Wasserstein geometry.

\item[\small\textbullet] From the perspective of learning theory, informative generalization bounds are derived based on the Wasserstein distance. It bridges the gap between the domain-level attention mechanism and the generalization error in UDA, and ensures that DoT is effective in learning transferable features.

\item[\small\textbullet] DoT is free of pseudo-labels in the target domain, which avoids the uncertainty in the pseudo-labeling procedure, and it can be used as a plug-and-play module for different network architectures. Extensive experiment results validate the effectiveness of DoT in tackling the UDA problem.
\end{itemize}

The rest of this paper is organized as follows. Section~\ref{sect:related-work} reviews existing UDA methods and transformers-based UDA advances. Section~\ref{sect:method} depicts the proposed DoT and its connection with the OT theory. Section~\ref{sect:experim} presents extensive experimental results and Section~\ref{sect:results-analysis} presents method analysis. Section~\ref{sect:conclusion} concludes the paper.

\section{Related Work}\label{sect:related-work}

In this section, we introduce some pioneering works on UDA and vision transformers, and compare the proposed DoT with these methods.

\subsection{Unsupervised Domain Adaptation}
Existing UDA methods mainly learn domain-invariant features under different assumptions.

Based on the covariate shift prior, explicit discrepancy optimization and adversarial training are the most common methods for marginal distribution alignment. The explicit discrepancy optimization methods align statistic moments of the source and target domains, e.g., MMD-based methods~\citep{pan2010domain,long2019transferable}and covariance-based method~\citep{sun2016return}. For adversarial training, \cite{ganin2016domain} provide Domain Adversarial Neural Networks (DANN), which minimize the loss of the category classifier and maximize the loss of the domain classifier.

Based on the conditional shift assumption, an increasing number of works attempt to learn locality consistency by incorporating the label information. Since the target domain is unlabeled, the pseudo-labeling technique is widely used in these methods. \citet{long2018conditional} define a domain discriminator conditioned on the classifier predictions in Conditional Domain Adversarial Network (CDAN). \citet{Pan_2019_CVPR} propose a Transferable Prototypical Network (TPN), which minimizes the distances between prototypes of the source domain with labels and target domain with pseudo-labels. \citet{luo2022unsupervised} propose a Discriminative Manifold Propagation (DMP) method, which generalizes Fisher's discriminant criterion by exploiting the between-class scatter in the source domain and the within-class scatter in the target domain. \citet{zhu2021deep} provide a local MMD, then match class-relevant subdomains by proposing a Deep Subdomain Adaptation Network (DSAN). \citet{ATM} propose a Maximum Density Divergence, then provide an Adversarial Tight Match (ATM) method based on this metric to maximize the inter-domain divergence and intra-class density. Based on pseudo-labels, \citet{kumar2023improving} swap the source image style with that of the target image using Class Aware Frequency Transformation (CAFT). \citet{yang2022heterogeneous} propose a graph attention network, which transfers the local structure information by aligning the class-wise centroids across domains. To identify transferable image regions, \citet{wang2019transferable} present Transferable Attention for Domain Adaptation (TADA), where the local attention is generated by a patch-wise domain discriminator.

Another fruitful line for UDA methodology is based on OT, which has a solid theoretical foundation in graph theory and optimization. \citet{courty2016optimal} align domains by solving the optimal transport plan of an entropy-regularized OT model. \citet{zhang2019optimal} provide a theoretical framework of OT in reproducing kernel Hilbert spaces, then match the domain distributions via a kernel-based transport map. These OT-based methods tend to align domains via pair-wise matching between cross-domain samples while ignoring the real influence of label information. To make full use of the label-based locality structure, \citet{jdot2017} propose the Joint Distribution Optimal Transport (JDOT) and \citet{Damodaran_2018_ECCV} extend it to the deep learning framework as DeepJDOT. They are the first attempt to seek the optimal classifier for the target domain by minimizing the optimal transport distance between the joint distributions. \citet{li2020enhanced} propose an Enhanced Transport Distance (ETD) method, which reweighs the transport distance with the similarity between feature vectors from different domains. \citet{RWOT2020} present Reliable Weighted Optimal Transport (RWOT) via exploiting spatial prototypical information and intra-domain strcuture. \citet{xia2020structure} leverages Gromov-Wasserstein distance to extract cross-domain matching relation. Inspired by~\citet{zhang2019optimal}, \citet{BuresNet2023} propose the Conditional Kernel Bures metric in BuresNet for characterizing the class-wise domain discrepancy. 

Some methods attempt to deal with UDA via exploring the learning capability of vision-language models like CLIP~\citep{CLIP}. \citet{Padclip} focus on the catastrophic forgetting issue of CLIP in adaptation learning process, and propose a dynamically adjust strategy for learning rate to avoid excessive training and overcome the performance degradation in knowledge transfer.  Instead of performing adaptation on the visual or textual modality independently, \citet{li2024memonav} unify the language-associated and vision-associated knowledge in CLIP, where the domain gap is effectively reduced by the developed modality-ensemble training method.

Different from current UDA models, which usually introduce explicit domain discrepancy optimization and pseudo-label procedure, DoT learns the locality consistency and reduces the generalization error via the domain-level attention automatically. DoT is more than an attention-based re-presentation method since we focus on the theoretical analysis of the domain-level attention mechanism. By drawing on OT theory and statistical learning theory, the effectiveness and interpretability of DoT can be guaranteed.

\subsection{Transformer-based UDA methods}
In NLP, the standard Transformer~\citep{vaswani2017attention} and its variants~\citep{BigBird2020,dosovitskiy2021an,Liu_swin} are built upon the attention mechanism, which maps a query $\bm{q}$ and a set of key-value pairs ($\bm{k},\bm{v}$) to formulate a series of matching probabilistic weights. The probabilistic weight assigned to each value depends on the similarity between the query and the corresponding key. To extend the Transformers to images-based learning tasks, Vision Transformer (ViT)~\citep{dosovitskiy2021an} and its variants~\citep{Liu_swin,DeiT} usually split each image into $N$ patches, which are then formulated as feature matrix $\mathbf{Z} \in \mathbb{R}^{N \times d}$. Different projections are applied on $\mathbf{Z}$ to obtain $\mathbf{Q}$, $\mathbf{K}$ and $\mathbf{V} \in \mathbb{R}^{N \times d_h}$. $d$ and $d_h$ indicate their dimensions. Then, the self-attention mechanism in vision transformers is formulated as
\begin{align}
\label{eq:att-transformer}
\begin{aligned}
  & \mathbf{A}  = \mbox{Softmax}\left(\frac{\mathbf{Q}\mathbf{K}^T}{\sqrt{d_h}}\right)\in {\mathbb{R}^{N \times N}},\\
  & \mbox{Attention}(\mathbf{Z})  = \mathbf{A}\mathbf{V}.
\end{aligned}
\end{align}
Such a formulation builds the long-range dependency among \textit{patches} of each image, which allows ViT to identify the image regions semantically relevant to classification. Furthermore, the formulation above induces the cross-attention mechanism when queries and keys are computed from different images~\citep{zhu2021DETR}.

\begin{figure*}[htb]
\centering{\scalebox{0.52}
{\includegraphics{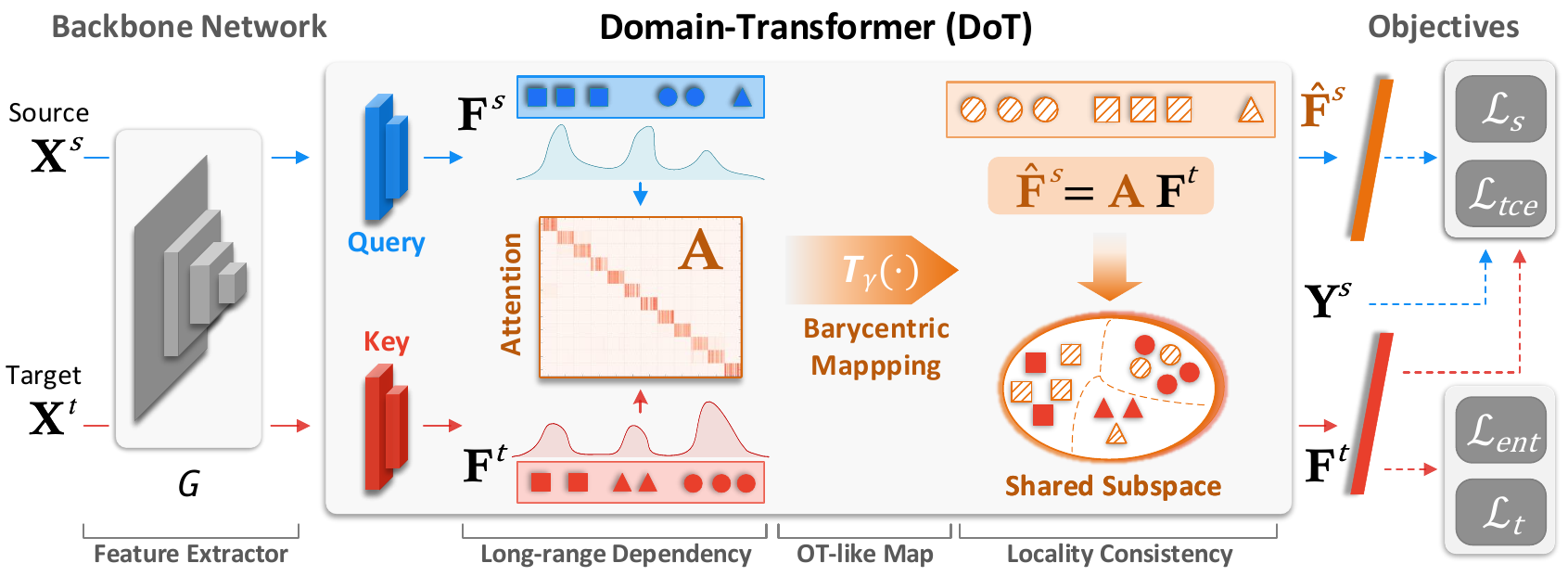}}
\caption{Illustration of our method, which consists of three parts, i.e., the shared backbone network, the domain-transformer (DoT), and the shared classifier. In particular, DoT can be further decomposed into the long-range dependency characterization, the OT-like map, and the locality consistency learning phases. It essentially explores the sample correspondence between domains by the novel attention mechanism at domain-level. With the projection from the source (e.g., $\mathbf{F}^s$) onto the subspace of the target domain, the new feature vectors (e.g., $\hat{\mathbf{F}}^s$) are weighted by the most relevant ones in the target domain, and thus, they are expected to be transferable and adaptive across domains. Note that $\mathbf{Y}^s$ represents the ground-truth label of source samples $\mathbf{X}^s$. Best viewed in color.}
\label{fig:DoT}}
\end{figure*}

Inspired by the successes of ViT, several methods have tried to address UDA with vision transformers. \citet{ma2021WinTR} propose the Win-Win Transformer (WinTR) to exploit domain-specific and invariant knowledge, which transfers cross-domain knowledge by target pseudo-label refinement and contrastive learning. To learn transferable and discriminative features, \citet{yang2023tvt} propose a Transferable Vision Transformer (TVT), which injects class-token and patch-token domain discriminators into the attention mechanism of ViT. \citet{xu2022cdtrans} propose a triple-branch Cross-Domain Transformer (CDTrans) for UDA, where the source-target branch adopts the cross-attention mechanism to align each cross-domain image pair via patch similarity. Based on an adversarial learning framework, \citet{SSRT2022} develop a Safe Self-Refinement for Transformer-based domain adaptation (SSRT), which utilizes perturbed target predictions to refine the model. \citet{SDAT2022} demonstrate that a smooth combination of task loss and adversarial loss leads to sub-optimal results, and then propose Smooth Domain Adversarial Training (SDAT) that only focuses on smoothing task loss. \citet{PMTrans2023} propose a PatchMix Transformer (PMTrans), which constructs an intermediate domain and interprets the process of domain alignment as a min-max cross entropy game.

However, these methods focus on the application of vision transformers while ignoring their interpretability. Besides, they directly employ the attention mechanism at patch-level, which cannot build interactions between cross-domain samples and fails to guarantee the shift minimization for cross-domain distributions. To overcome these challenges, inspired by ViT, we propose the domain-level attention mechanism. Methodologically, it captures the sample correspondence and promotes locality consistency across domains, which are indeed important for the deep backbones that only consider patch-level correspondence. Theoretically, we build a theoretical connection between attention and OT geometry, and derive informative generalization bounds based on OT geometry. This connection offers a new perspective for understanding Transformer, and presents a reasonable interpretation for DoT, i.e., the domain shift and generalization error are explicitly connected with DoT.

\section{Methodology}\label{sect:method}
In this section, we introduce the Domain-Transformer (DoT) for UDA. Section \ref{subsec:Domain-Transformer} details the core mechanism of DoT, i.e., the domain-level cross-attention. Section~\ref{subsec:connection_OT_DoT} builds the connection between DoT and OT theory. Section~\ref{subsec:theoretical_analysis} provides the generalization error bounds. Finally, the local structure regularization and the overall algorithm of DoT are presented in Section~\ref{subsec:local_structure_reg} and Section~\ref{subsection:algorithm-and-optimization}, respectively.

In UDA, we assume that we have access to a labeled source domain $\mathcal{D}^{s}=\{(\bm{x}_i^s, y_i^s)\}_{i = 1}^{n_s}$ and an unlabeled target domain $\mathcal{D}^t = \{\bm{x}_j^t\}_{j = 1}^{n_t}$, where $\bm{x}_i^s$, $\bm{x}_j^t$ represent samples and $y_i^s \in \{1,2,\cdots, K\}$ denotes the ground-truth label of $\bm{x}^s_i$.

\subsection{Domain-Transformer}
\label{subsec:Domain-Transformer}

We first propose the DoT method with the domain-level cross-attention mechanism, which is helpful to build sample correspondence across domains.

Different from the explicit position information of sentences in NLP, samples from the two domains are position-free. We can explore the implicit local structures across domains in the latent feature space. With the feature extractor $G(\cdot)$, the feature vectors of the source and target domains are represented by
\begin{equation}
\label{Eq:deep-features}
\begin{aligned}
\mathbf{G}^s = G(\mathbf{X}^s)~~~\mbox{and}~~~\mathbf{G}^t = G(\mathbf{X}^t),\\
\end{aligned}
\end{equation}
respectively, where $\mathbf{G}^s\!\in\! \mathbb{R}^{n_s \times d_1 }$ and $\mathbf{G}^t\! \in\! \mathbb{R}^{n_t \times d_1 }$. Note that the feature extractor can employ CNNs or vision transformers as a backbone network.

Under the assumption of UDA, samples from different domains may have potentially significant distances due to the domain shifts. In this paper, we refer to the similarities between samples within the same underlying class but different domains as lone-range correspondence. Besides, we formulate this long-range dependency by proposing the domain-level attention, in which the \textit{queries} come from the source domain and the \textit{key-value pairs} are derived from the target domain. To simulate the separate and learnable feature embeddings of queries and key-value pairs, we employ the fully-connected networks (FCNs) for projections. These two FCNs for the source domain and the target domain are denoted as $F_s(\cdot)$ and $F_t(\cdot)$, respectively, which can be embedded into the whole network architecture and then trained in an end-to-end manner. The \textit{queries} and \textit{keys} (\textit{values}) can be formulated by
\begin{equation}
\label{Eq:FCNs}
\mathbf{F}^s = F_s(\mathbf{G}^s),~~~~~\mathbf{F}^t = F_t(\mathbf{G}^t),
\end{equation}
where $\mathbf{F}^s = [\bm{f}_1^s,\bm{f}_2^s,\cdots,\bm{f}_{n_s}^s]^T \in \mathbb{R}^{n_s \times d_2 }$ and $\mathbf{F}^t \in \mathbb{R}^{n_t \times d_2 }$.

As shown in the middle of Fig.~\ref{fig:DoT}, DoT exchanges locality information of samples across domains. This procedure is achieved by the domain-level attention mechanism, which formulates the attention map $\mathbf{A}\in \mathbb{R}^{n_s \times n_t}$ (also known as attention score matrix) as the expression coefficients across domains, \ie,
\begin{equation}
\label{Eq:softmax-DoT}
\mathbf{A} = \mbox{softmax}\left(\frac{\mathbf{F}^s{\mathbf{F}^t}^T}{\sqrt{d_2}}\right),
\end{equation}
or
\begin{equation}
\label{Eq:entries-of-A}
a_{ij} = \frac{\exp({\bm{f}_i^s}^T \bm{f}_j^t/\sqrt{d_2})}{\sum_{i = 1}^{n_t}\exp({\bm{f}_i^s}^T \bm{f}_j^t /\sqrt{d_2})}.
\end{equation}
The softmax function is applied to the rows of the scaled similarity matrix $\frac{\mathbf{F}^s{\mathbf{F}^t}^T}{\sqrt{d_2}}$, which characterizes the pair-wise similarities between the source sample $\mathbf{X}^s$ and the target sample $\mathbf{X}^t$ in the feature subspace. Specifically, the $i$-th row of $\mathbf{A}$ represents the similarities between \textit{query} $\bm{f}_i^s$ and all \textit{keys} $\{\bm{f}_j^t\}_{j = 1}^{n_t}$. Thus, it can be viewed as a competitive expression of sample correspondence intensity. If the attention weight $a_{ij}$ is larger, the source feature $\bm{f}_i^s$ will allocate more attention to the target feature $\bm{f}_i^t$. The output representations of the domain-transformer module, i.e., the transformed source features (also called transformed \textit{queries}) $\hat{\mathbf{F}}^s$, can be represented by
\begin{equation}\label{eq:ot-map}
    \hat{\mathbf{F}}^s = \mathbf{A}\mathbf{F}^t\in \mathbb{R}^{n_s \times d_2}.
\end{equation}

For each \textit{query} $\bm{f}^s$, it can be observed that the transformed \textit{query} $\hat{\bm{f}}^s_i$ is a weighted sum of \textit{keys} $\{\bm{f}^t_j\}_{j=1}^{n_t}$. Among all the feature vectors $\{\bm{f}^t_j\}_{j=1}^{n_t}$ in the target domain, those more similar to \textit{query} $\bm{f}^s_i$ in the source domain will have a larger attention weight $a_{ij}$, and they contribute more to the transformed \textit{query} $\hat{\bm{f}}^s_i$. Though there is a discrepancy between source and target domains, feature vectors in the same class different domains are expected to be more similar. Thus, each transformed \textit{query} $\hat{\bm{f}}^s_i$ is mainly weighted by the most relevant \textit{values/keys} in $\{\bm{f}^t_j\}_{j=1}^{n_t}$. In this perspective, the attention map $\mathbf{A}$ characterizes long-range dependency and captures the correspondence between samples in the same class but different domains. Thus, the transformed \textit{queries} $\hat{\mathbf{F}}^s$ can be considered as the image of the source features $\mathbf{F}^s$ in the subspace of the target domain. Since $\hat{\mathbf{F}}^s$ retains labels of $\mathbf{F}^s$, the classifier $C(\cdot)$ trained on $\hat{\mathbf{F}}^s$ is expected to be generalized well to the target domain.

\subsection{Connection between DoT and OT}
\label{subsec:connection_OT_DoT}

Both the similarity estimation approach and the projection manner, as shown in Eq.~\eqref{Eq:entries-of-A} and Eq.~\eqref{eq:ot-map}, further encourage us to explore the connection between DoT and the OT algorithm. The relationship will provide new insight into understanding the DoT method from the view of Wasserstein geometry. Moreover, it implies that the domain discrepancy can be explicitly mitigated in the DoT model. This is just very important to deal with the UDA problem effectively.

In the context of domain adaptation, once the attention map $\mathbf{A}$ has been computed, the sample correspondence between the source and the target domains can be considered as determined.
It is interesting that the attention map $\mathbf{A}$ shown in Eq.~\eqref{Eq:softmax-DoT} has an intrinsic connection with the transport/assignment plan $\gamma$ in OT.
Here we denote $\Omega\! =\! \{\bm{\gamma}\in \mathbb{R}_+^{n_s \times n_t}~|~\bm{\gamma}\mathbf{1}_{n_t} \!=\! \mu_s, ~\bm{\gamma}^T\mathbf{1}_{n_s} \!= \!\mu_t\}$, where $\mu_{s/t}$ denotes the empirical distribution of the source/target domain. The entropy-regularized Kantorovich formulation of the OT problem \citep{cuturi2013sinkhorn} is defined as
\begin{align}
\label{Eq:regularized-ot}
\underset{\bm{\gamma}\in\Omega}{\arg\min}\left< \bm{\gamma}, \mathbf{M} \right>_F + \lambda H(\bm{\gamma}),
\end{align}
where $\left<\cdot, \cdot \right>_F$ is the Frobenius dot product, $\mathbf{M}$ is the cost matrix, and $H(\bm{\gamma}) = \sum_{i,j}\bm{\gamma}_{ij}\log \bm{\gamma}_{ij}$. Generally, the squared $\ell_2$ distance, \ie, $\mathbf{M}_{ij} = \|\bm{f}_i^s - \bm{f}_j^t \|^2_2$, is used to compute the cost matrix.
For $\lambda>0$, the OT solution $\bm{\gamma}^{\ast}$ of Eq.~\eqref{Eq:regularized-ot} is unique and has the form~\citep{cuturi2013sinkhorn}
\begin{align}
\label{Eq:solution-for-ot}
\bm{\gamma}^{\ast} = \mbox{diag}(\mu)\exp\left(-\frac{\mathbf{M}}{\lambda}\right)\mbox{diag}(\nu),
\end{align}
where $\mu\in\mathbb{R}^{n_s}$ and $\nu \in \mathbb{R}^{n_t}$ are two none-negative vectors. It can be recast into a bi-stochastic matrix scaling problem \citep{wang2022robust} and then solved by the Sinkhorn-Knopp Algorithm \citep{knight2008sinkhorn}.
We can rewrite Eq.~\eqref{Eq:solution-for-ot} as
\begin{align}
\bm{\gamma}^{\ast}_{ij}& = \mu_i\exp(-\mathbf{M}_{ij}/{\lambda})\nu_j \notag \\
& =\mu_i\exp(-\|\bm{f}_i^s - \bm{f}_j^t\|_2^2/{\lambda})\nu_j \notag \\
& \propto \mu_i\exp(2(\bm{f}_i^s)^T\bm{f}_j^t/{\lambda})\nu_j. \label{Eq:solution-ot-rewrite}
\end{align}

Compared Eq.~\eqref{Eq:solution-ot-rewrite} with Eq.~\eqref{Eq:entries-of-A}, it can be observed that both $a_{ij}$ of the attention map and $\gamma^{\ast}_{ij}$ of the OT plan are based on the matrix scaling problem of $\exp({\bm{f}^s_i}^T\bm{f}_j^t$).
More specifically, by setting $\lambda=2\sqrt{d_2}$, it is row-stochastic for the attention map $\mathbf{A}$, while doubly-stochastic for the OT plan due to its double constraints on marginal distributions. Thus, $\bm{\gamma}^{\ast}$ is empirically much sparser than $\mathbf{A}$. Intrinsically, columns of $\hat{\mathbf{F}}^s$ are sparsely linear combinations of those in $\mathbf{F}^t$. In this perspective, the attention map $\mathbf{A}$ has some equivalence to barycentric mapping induced by the optimal transport plan $\bm{\gamma}$. Therefore, some appealing properties of the optimal transport literature can be explored in our DoT method. Suppose that the marginal distribution (w.r.t. features) of the source domain and the target domain is $\mu_s$ and $\mu_t$, respectively. We can interpolate the two distributions by the geodesics of the Wasserstein metric~\citep{Villani-09}, parameterized by a scalar $t\in [0,1]$. This defines a new or intermediate distribution $\hat{\mu}$ such that
\begin{equation}\label{eq:intermediate-distr}
\hat{\mu} = \underset{\mu}\arg\min~t\mathcal{W}^2_2(\mu_s,\mu)+(1-t)\mathcal{W}^2_2(\mu_t,\mu).
\end{equation}
This distribution boils down to~\citep{Villani-09}
\begin{align}\label{eq:intermediate-distr-explicit}
\hat{\mu} = \sum_{i,j} \gamma_{ij}\delta_{t\bm{f}^s_i+(1-t)\bm{f}^t_j}.
\end{align}
As a result, we can map samples in the source domain onto the target domain by selecting $t=1$. Then, the intermediate distribution $\hat{\mu}$ is a distribution with the same support of $\mu_t$, and we have $\hat{\mu}=\sum_j\hat{p}^t_j\delta_{\bm{f}^t_j}$,
with $\hat{p}^t_j=\sum_i \gamma_{ij}$. Note the weights $\hat{p}^t_j$ is the sum of weights between all source samples $\bm{f}^s_i$ and the target sample $\bm{f}^t_j$. Recall that $a_{ij}$ tells us how much probability mass of $\bm{f}^s_i$ is transferred to $\bm{f}^t_j$, by comparing the attention matrix with the optimal transport plan. Along this line, we can exploit the domain-level attention as the following barycentric mapping based on the equivalence in Eq.~\eqref{Eq:solution-ot-rewrite}, i.e.,
\begin{equation}\label{eq:barycentric-map-ot}
\hat{\bm{f}}^s_i = \underset{\bm{z}\in \mathbb{R}^{d_2}}{\arg\min}\sum_{j} \gamma_{ij}m(\bm{z}, \bm{f}^t_j),
\end{equation}
where $m$ denotes the distance metric, $\bm{f}^t_j$ is a given feature vector in the target domain and $\hat{\bm{f}}^s_i$ is its transformed representation. In particular, when the distance function is the squared $\ell_2$ metric, this barycenter corresponds to a weighed average and the sample is mapped into the convex hull of the target samples~\citep{courty2016optimal}. For all samples in the source domain, the barycentric mapping has an analytic formulation as
\begin{equation}\label{eq:barycentric-map}
\hat{\mathbf{F}}^s = \bm{T}_{\gamma}(\mathbf{F}^s)=\mbox{diag}(\bm{\gamma}\mathbf{1}_{n_t})^{-1}\bm{\gamma}\mathbf{F}^t=\mathbf{A}\mathbf{F}^t,
\end{equation}
where $\mathbf{1}_{n_t}$ is the all-ones vector in $\mathbb{R}^{n_t}$, and $\bm{T}_{\gamma}$ denotes the optimal transport map associating with the transport plan $\bm{\gamma}$. This result has multiple important meanings in domain adaptation. It is a first-order approximation of the true $n_s$ Wasserstein barycenters of the target distribution, as stated by~\citep{courty2016optimal}. Moreover, the projection matrix $\mbox{diag}(\bm{\gamma}\mathbf{1}_{n_t})^{-1}\bm{\gamma}$ is the (row)-normalized similarity weights in the attention mechanism, as shown in Eq.~\eqref{eq:ot-map}. Therefore, it builds an intrinsic connection between DoT and the optimal transport theory, and thus presents a reasonable interpretation for the effectiveness of our DoT method.

\subsection{Theoretical Analysis}
\label{subsec:theoretical_analysis}

Based on the connection between DoT and OT algorithms, an informative generalization bound for UDA can be derived from the perspective of learning theory. The theoretical results provide insights into understanding the proposed domain-attention mechanism in UDA problem.

To theoretically analyze the generalization performance of DoT method, a generalization bound with the Wasserstein metric-based discrepancy is necessary. We consider the domain adaptation problem in the bi-class setting, with continuous outputs, as follows. To facilitate the characterization of distribution gap and labeling discrepancy, we define a domain as a pair consisting of a distribution $P_X$ on input $\MC{X}$ and a labeling function $f:\MC{X}\rightarrow [0,1]$, which can have a fractional (expected) value when labeling occurs non-determinedly. More specifically, we denote by $(P_X^s, f_s)$ the source domain and $(P_X^t, f_t)$ the target domain. A hypothesis space $\MC{H}$ is a set of hypotheses $h:\MC{X}\rightarrow [0,1]$.

\begin{definition}[True Risk]
The true risk is defined as the probability according to the distribution $P_X^s$ that a hypothesis $h$ disagrees with a labeling function $f_s$ (which can also be a hypothesis), \ie,
\begin{equation}
\epsilon_s(h,f_s):=\MBB{E}_{x\sim P_X^s}[\left|h(x)-f_s(x)\right|].
\end{equation}
\end{definition}

Note that the true risk satisfies the triangle inequality, i.e., $\epsilon_s(h_1,h_3)\leq \epsilon_s(h_1,h_2)+ \epsilon_s(h_2,h_3)$, for any hypotheses $h_1$, $h_2$ and $h_3\in \MC{H}$. When we want to refer to the source error of a hypothesis, we use the shorthand $\epsilon_s(h)=\epsilon_s(h,f_s)$. Without loss of generality, we use the parallel notations $\epsilon_t(h,f_t)$ and $\epsilon_t(h)$ for the target domain.

Then we introduce some preliminary definitions on the probability metric, which will be used to reformulate the risk function as the (domain) distribution discrepancy.

\begin{definition}[Seminorm]
Let $\MC{F}$ be a function space. A real valued function $\rho:\MC{F}\rightarrow\MBB{R}$ is called a seminorm if it satisfies the following two conditions.
\begin{itemize}
  \item \textit{Subadditivity inequality:} $\rho(f_1+f_2)\leq \rho(f_1)+\rho(f_2)$ for any $f_1,f_2\in\MC{F}$;
  \item \textit{Absolute homogeneity:} $\rho(af)=|a|\rho(f)$ for any $f\in\MC{F}$ and all scalars $a$.
\end{itemize}
\end{definition}

Note that these two conditions imply that $\rho(0)=0$ and that every seminorm $\rho$ also has the \textit{Nonnegativity} property: $\rho(f)\geq 0$ for any $f\in \MC{F}$.

\begin{definition}[Sobolev Seminorm]
Let $(\MC{X},d)$ be a metric space. For any real valued function $f$ on $\MC{X}$, its Sobolev seminorm is defined as
\begin{equation}
\|f\|_{H^1}:=(\int\|\nabla f(x)\|^2dx)^{\frac{1}{2}}.
\end{equation}
\end{definition}

\begin{definition}[Integral Probability Metric (IPM)]
Let $\MC{F}$ be a class of real-valued bounded measurable functions on $\MC{X}$. The IPM between two distributions $P_X^s$ and $P_X^t$ is defined as
\begin{equation}\label{eq:IPM-def}
d(P_X^s, P_X^t) := \mathop{\sup}_{f\in\MC{F}} \Big| \MBB{E}_{x\sim P_X^s}[f(x)] - \MBB{E}_{x\sim P_X^t}[f(x)] \Big|.
\end{equation}
\end{definition}

Choosing $\MC{F}=\{f: \| f \|_{H^1} \leq 1\}$ in Eq.~\eqref{eq:IPM-def} yields
\begin{eqnarray}
\label{eq:IPM-sobolev}
\|P_X^s- P_X^t\|_{H^{-1}} :=& \mathop{\sup}_{f:\|f\|_{H^1}\leq 1} \Big| \MBB{E}_{x\sim P_X^s}[f(x)]   \notag\\
&-\MBB{E}_{x\sim P_X^t}[f(x)] \Big|.
\end{eqnarray}
Furthermore, there exists two positive constants $c_1$ and $c_2$ such that~\citep{sriperumbudur2009integral,santambrogio2015optimal}
\begin{equation}
 c_1\|P_X^s\!-\! P_X^t\|_{H^{-1}} \!\leq \!\MC{W}_2(P_X^s, P_X^t)\! \leq\! c_2\|P_X^s\!-\! P_X^t\|_{H^{-1}}.
\end{equation}
It indicates that $H^{-1}$ norm, which is the dual norm of the Sobolev seminorm, is equivalent to the 2-Wasserstein $\MC{W}_2$ distance. In other words, the $\MC{W}_2$ distance gives an upper bound on an integral probability metric in $\MC{F}$.

\begin{remark}
This equivalence implies that the disagreement of any hypothesis on two feature distributions is upper-bounded by the 2-Wasserstein distance between them. Thus, the inequality can be employed to reformulate the hypotheses' distance in risk objective as $\MC{W}_2$ between feature distributions, which is actually the discrepancy between transformed features in Eq.~\eqref{eq:ot-map}.
\end{remark}

Next we show the generalization error bounded by $\MC{W}_2$ distance across domains. Generally, we assume the hypothesis space $\MC{H}= \MC{F}$, which is reasonable for regression task or binary classification with probability output. First, we bound the distance between the true risks across domains as the following lemma.

\begin{lemma}\label{lem:W1-bound}
For any hypotheses $h_1,h_2\in\MC{H}$, there exists a positive constant $c$ such that
\begin{equation}
\left| \epsilon_s(h_1,h_2) - \epsilon_t(h_1,h_2) \right| \leq c\MC{W}_2(P_X^s, P_X^t).
\end{equation}
\end{lemma}

\begin{proof}
For any $h_1,h_2\in\MC{H}$, denote $\bar{h} = \frac{h_1-h_2}{2}$. Note that $\bar{h}$ still belongs to $\MC{H}$ since
\begin{equation}
  \| \bar{h} \|_{H^1} = \left\|\frac{h_1-h_2}{2}\right\|_{H^1}\leq\frac{\| h_1\|_{H^1}+\| h_2\|_{H^1}}{2}\leq 1.
\end{equation}
Now we prove the main inequality.
\begin{eqnarray}
&~& \left| \epsilon_s(h_1,h_2) - \epsilon_t(h_1,h_2) \right| \notag\\
&\leq& 2 \Big|\MBB{E}_{x\sim P_X^s}\!\left[\frac{h_1(x)\!-\!h_2(x)}{2}\right] \notag\\
&~&-\MBB{E}_{x\sim  P_X^t}\!\left[\frac{h_1(x)\!-\!h_2(x)}{2}\right] \Big| \notag\\
&=& 2\Big| \MBB{E}_{x\sim P_X^s}\left[\bar{h}(x)\right] - \MBB{E}_{x\sim  P_X^t}\left[\bar{h}(x)\right] \Big| \notag\\
&\leq& 2\mathop{\sup}_{h\in\MC{H}} \Big| \MBB{E}_{x\sim P_X^s}[h(x)] - \MBB{E}_{x\sim P_X^t}[h(x)] \Big| \notag\\
&=& 2\|P_X^s- P_X^t\|_{H^{-1}} \notag\\
&\leq& c\MC{W}_2(P_X^s, P_X^t).
\end{eqnarray}
\end{proof}


Based on this lemma, we now connect the generalization error with the $\MC{W}_2$ distance (similar to~\cite[Theorem 2]{ben2010theory}). Note that the above lemma cannot be directly applied to labeling function $f_s$ and $f_t$, since they may not be a hypothesis. Usually, to ensure the tightness of the upper bound, we define an intermediate hypothesis $h^*$ which minimizes the joint error as
\begin{equation}
h^* = \mathop{\arg \inf}_{h\in\MC{H}} \epsilon_s(h)+ \epsilon_t(h).
\end{equation}
We denote the optimal joint error of the intermediate hypothesis $h^*$ as $\lambda = \epsilon_s(h^*)+ \epsilon_t(h^*)$.

\begin{theorem}
\label{thm:GUB-optimal-hypothesis}
Assume that $\MC{H}$ is a hypothesis space satisfying $\| h\|_{H^1}\leq 1$ for any $h\in\MC{H}$. Let $P_X^s$ and $P_X^t$ be the marginals over $\MC{X}$ of the source and target domains, respectively. Then for any $h\in\MC{H}$, there exists a positive constant $c$ such that
\begin{equation}
\epsilon_t(h)\leq \epsilon_s(h) + c\MC{W}_2(P_X^s, P_X^t) + \lambda.
\end{equation}
\end{theorem}

\begin{proof}
\begin{eqnarray}
\epsilon_t(h)&=&\epsilon_t(h,f_t) \notag \\
&\leq & \epsilon_t(h,h^*) + \epsilon_t(h^*,f_t) \notag\\
&\leq & \epsilon_s(h,h^*) + c\MC{W}_2(P_X^s, P_X^t) + \epsilon_t(h^*,f_t) \notag\\
&\leq & \epsilon_s(h,f_s) + \epsilon_s(h^*,f_s) + c\MC{W}_2(P_X^s, P_X^t) \notag\\
&~&+\epsilon_t(h^*,f_t) \notag\\
& \leq & \epsilon_s(h) + c\MC{W}_2(P_X^s, P_X^t) + \lambda.
\end{eqnarray}
\end{proof}

Theorem~\ref{thm:GUB-optimal-hypothesis} shows that the generalization is bounded by the $\MC{W}_2$ distance across domains. Note that $\MC{W}_2$ distance is generally smaller than the $\MC{H}\Delta\MC{H}$ divergence in~\cite[Theorem 2]{ben2010theory}, where $\MC{H}$ can be any function space. Besides, $\MC{W}_2$ distance can be explicitly estimated in the OT framework. Further, if the labeling functions $f_s$ and $f_t$ belong to the hypothesis space $\MC{H}$, we can deduce a more informative upper bound without the intractable constant $\lambda$.


\begin{theorem}
\label{thm:GUB-labeling-dist}
Assuming $\MC{H}$ is a hypothesis space satisfies that $\| h\|_{H^1}\leq 1$ for any $h\in\MC{H}$, and $f_s, f_t\in \MC{H}$. Let $P_X^s$ and $P_X^t$ be the marginals over $\MC{X}$ of the source and target domains, respectively. Then for any $h\in\MC{H}$, there exists a positive constant $c$ such that
\begin{equation}
\begin{aligned}
\epsilon_t(h)~\leq~& \epsilon_s(h) + c\MC{W}_2(P_X^s, P_X^t) \\
 &+ \min\left\{ \epsilon_s(f_s,f_t),\epsilon_t(f_s,f_t) \right\}.
\end{aligned}
\end{equation}
\end{theorem}

\begin{proof}
If we apply Lemma \ref{lem:W1-bound} first, then
\begin{eqnarray}
\epsilon_t(h)&=& \epsilon_t(h,f_t) \nonumber\\
& \leq & \epsilon_s(h,f_t) + c\MC{W}_2(P_X^s, P_X^t) \nonumber\\
& \leq & \epsilon_s(h,f_s) + c\MC{W}_2(P_X^s, P_X^t) +\epsilon_s(f_s,f_t).
\end{eqnarray}
If we apply triangle inequality to the true risk first, then
\begin{eqnarray}
\epsilon_t(h) & = &  \epsilon_t(h,f_t) \nonumber\\
& \leq &  \epsilon_t(h,f_s) + \epsilon_t(f_s,f_t) \nonumber\\
& \leq &  \epsilon_s(h,f_s) + c\MC{W}_2(P_X^s, P_X^t) + \epsilon_t(f_s,f_t).
\end{eqnarray}
By combing the inequality above, we have
\begin{equation}
\begin{aligned}
\epsilon_t(h)~\leq~ & \epsilon_s(h) + c\MC{W}_2(P_X^s, P_X^t)  \\
& + \min \left\{ \epsilon_s(f_s,f_t), \epsilon_t(f_s,f_t) \right\}.
\end{aligned}
\end{equation}
\end{proof}

\begin{remark}
Under the coviariate shift assumption, the labeling functions of different domains are the same, i.e., $f=f_s=f_t$. Then the bound of Theorem \ref{thm:GUB-labeling-dist} boils down to
\begin{equation}
\left|\epsilon_s(h) - \epsilon_t(h)\right| \leq  c\MC{W}_2(P_X^s, P_X^t).
\end{equation}
It implies that minimizing the Wasserstein distance across domains is sufficient for successful domain adaptation.
\end{remark}

Note that the domain-level transformer in DoT is mathematically analogous to the barycentric mapping in OT. Then, the generalization upper-bounds in Theorem~\ref{thm:GUB-optimal-hypothesis} and Theorem~\ref{thm:GUB-labeling-dist} implies that the true risk on target domain can be bounded by the DoT model. These results ensure that the domain-level transformer module can explicitly reduce the domain discrepancy and enhance the model's transferability with the OT-like mapping Eq.~\eqref{eq:barycentric-map}.

Further, we consider the upper bound with empirical estimator. Denote $\mathbf{X}^s$ and $\mathbf{X}^t$ be the samples drawn according $P_X^s$ and $P_X^t$ with sample size $n$, respectively. Then the empirical distributions of $\mathbf{X}^s$ and $\mathbf{X}^t$ is denoted as $\hat{P}_X^s$ and $\hat{P}_X^t$, respectively.

\begin{theorem}[Rate of Wasserstein Distance]\citep{peyre2019computational}\
\label{thm:Wdistance-rate}
Let $\MC{X}=\MBB{R}^d$ and measures ($P_X^s, P_X^t$) are supported on bounded domain. For any $d>2$ and $1\leq p < +\infty$,
\begin{align}
& \MBB{E}_{(\hat{P}_X^s,\hat{P}_X^t)\sim (P_X^s,P_X^t)^n} \left[ |\MC{W}_p(\hat{P}_X^s,\hat{P}_X^t)\! -\! \MC{W}_p(P_X^s,P_X^t)| \right]  \nonumber \\
& \lesssim \MC{O}(n^{-\frac{1}{d}}).
\end{align}
\end{theorem}

The notation $f_1(n)\lesssim f_2(n)$ indicates that there exists a constant $\tilde{c}$, depending on $f_1$ and $f_2$ but not $n$, such that $f_1(n)\leq \tilde{c}f_2(n)$. Theorem \ref{thm:Wdistance-rate} implies that the empirical Wasserstein distance $\MC{W}_p(\hat{P}_X^s,\hat{P}_X^t)$ is asymptotically consistent. Based on this asymptotic property of Wasserstein distance, we could prove the generalization upper bound with empirical Wasserstein distance.

\begin{corollary}[Empirical Generalization Upper Bound]
\label{cor:GUB_empirical_case}
Assume that $\MC{H}$ is a hypothesis space satisfying $\| h\|_{H^1}\leq 1$ for any $h\in\MC{H}$. Let $P_X^s$ and $P_X^t$ be the marginals over $\MC{X}=\MBB{R}^d$ of the source and target domains, respectively. Denote $\hat{P}_X^s$ and $\hat{P}_X^t$ as the empirical distributions of samples drawn from $P_X^s$ and $P_X^t$ with sample size $n$, respectively. Then for any $h\in\MC{H}$, there exists a positive constant $c$ such that the following holds with high probability:
\begin{equation}
\epsilon_t(h)\leq \epsilon_s(h) + 2c\MC{W}_2(\hat{P}_X^s,\hat{P}_X^t) + \lambda + \MC{O}(n^{-\frac{1}{d}}).
\end{equation}
Furthermore, if $f_s, f_t\in \MC{H}$, then the following holds with high probability:
\begin{equation}
\begin{aligned}
\epsilon_t(h) ~\leq ~ & \epsilon_s(h) + 2c\MC{W}_2(\hat{P}_X^s,\hat{P}_X^t) + \MC{O}(n^{-\frac{1}{d}}) \\
     & + \min \left\{ \epsilon_s(f_s,f_t), \epsilon_t(f_s,f_t) \right\} .
\end{aligned}
\end{equation}

\end{corollary}
Corollary~\ref{cor:GUB_empirical_case} implies that the empirical generalization upper bound will converge to the truth with the increasing sample size. It ensures that the risk with finite samples is still bounded by the DoT model while the appealing property of DoT is also valid in empirical case.

\subsection{Local Structure Regularization}
\label{subsec:local_structure_reg}

The proposed domain-level attention module implicitly endows OT property and ensures feature transferability via its connection with the Wasserstein distance. To simultaneously ensure transferability and discriminability, we further explore local regularization terms for locality consistency learning. These terms are formulated by the Locality Preserving Projection (LPP) algorithm~\citep{he2003LPP}, which satisfies the requirements of both supervised and unsupervised learning. Essentially, the local regularization terms promote a class-sparse attention matrix, which enhances more reliable sample correspondence across domains.

\subsubsection{Unsupervised Local Structure Learning}
Since the target domain is unlabeled, we incorporate an unsupervised LPP constraint into the target's local regularization term. Firstly, we construct the adjacency graph via $k$ nearest neighbors with the Euclidean distance. Each pair of nodes is connected once one node is among $k$ nearest neighbours of the other node. Then, the adjacency matrix $\mathbf{W}^t$ is defined by a binary variable, \ie, ${w}^t_{ij} = 1$ if and only if the vertices $i$ and $j$ are connected by an edge, and ${w}^t_{ij} = 0$ otherwise. The objective of unsupervised local structure learning for the target domain is formulated as
\begin{equation}\label{eq:lpp_t}
\mathcal{L}_{t}(\bm{\theta}_G, \bm{\theta}_{F_t}) = \sum_{i,j}^{n_t}\|\bm{f}_i^t - \bm{f}_j^t\|_2^2{w}^t_{ij}.
\end{equation}

The entropy criterion is also utilized to explore intrinsic locality structure of the target domain. Mathematically, the entropy loss $\mathcal{L}_{ent}$ is formulated as
\begin{equation}
\label{eq:target-entropy-loss}
\mathcal{L}_{ent}(\bm{\theta}_G, \bm{\theta}_{F_t}, \bm{\theta}_C)= \frac{1}{n_t}\sum_{j=1}^{n_t}\sum_{k=1}^{K}-\hat{p}_{jk}^t\log \hat{p}_{jk}^t,
\end{equation}
where $\hat{p}_{jk}^t$$=$$C(F_t(G(\bm{x}_j^t)))$ satisfying to $\sum_{k=1}^K \hat{p}_{jk}^t\!=\! 1$ is the prediction probability of $\bm{x}_j^t$ belonging to the $k$-th class.

\subsubsection{Supervised Local Structure Preserving}

Generally, the domain-level attention encodes class-relevant target feature vectors for each source feature vector. It is desirable that similar samples in the source domain also be similar after the transformation of the domain-level attention. Since larger weights indicate higher similarities, those class-relevant target feature vectors are expected to share the same class label with the corresponding source feature vector. Then, the transformed source feature vector empirically share the same label with the original source feature vector. (We conduct experiments to verify this in the supplementary material.)

The labels of $\hat{\mathbf{F}}^s$ are the same with those of $\mathbf{F}^s$. In order to preserve discriminative structure of the source domain, we still use LPP on the transformed source domain and formulate the regularization term as
\begin{equation}\label{eq:lpp_s}
\mathcal{L}_{s}(\bm{\theta}_G,\bm{\theta}_{F_s}, \bm{\theta}_{F_t}) = \sum_{i,j}^{n_s}\|\hat{\bm{f}}_i^s - \hat{\bm{f}}_j^s\|_2^2{w}^s_{ij},
\end{equation}
where
\begin{equation}
{w}^s_{ij} = \left\{
\begin{array}{rcl}
1, &   & \mbox{if } y_i^s = y_j^s,\\
0, &   & \mbox{else}.
\end{array}\right.
\end{equation}
By minimizing the discriminant regularization term in Eq.~\eqref{eq:lpp_s}, samples from the same class are expected to be projected close to each other in the subspace.

\subsection{Algorithm and Optimization}\label{subsection:algorithm-and-optimization}

\begin{algorithm}[t]
\caption{DoT for UDA}
\label{algorithm1}
\begin{algorithmic}[1]
\Require{Source domain $\mathcal{D}^s\!=\!{\{\bm{x}_i^s, y_i^s\}}_{i = 1}^{n_s}$, target domain $\mathcal{D}^t\!=\!{\{\bm{x}_j^t\}}_{j = 1}^{n_t}$, batch-sizes $b_s$, $b_t$, and maximum iteration $N^{iter}$}
\Ensure{Networks parameters $\bm{\theta}$ and predictions $\{\hat{y}^t_j\}_{j = 1}^{n_t}$}
\State{Pre-train $G(\cdot)$, $F_s(\cdot)$ and $C(\cdot)$ via cross-entropy loss on $\mathcal{D}^s$}
\State{Initialize $F_t(\cdot)$ by the pre-trained $F_s(\cdot)$}
\For{n=1 to $N^{iter}$}
    \State{Randomly sample batch $\mathcal{B}^s\!=\!{\{\bm{x}_i^s, \bm{y}_i^s\}}_{i = 1}^{b_s}$}
    \Statex{$\quad$ and $\mathcal{B}^t\!=\!{\{\bm{x}_j^t\}}_{j = 1}^{b_t}$}
    \State{Obtain features $\mathbf{G}^s$ and $\mathbf{G}^t$ via Eq.~\eqref{Eq:deep-features}}
    \Statex{$\quad$\% \textbf{\textit{Domain-level attention}}}
    \State{Project $\mathbf{G}^s$ and $\mathbf{G}^t$ into queries $\mathbf{F}^s$ and}
    \Statex{$\quad$ keys/values $\mathbf{F}^t$, with different projections,}
    \Statex{$\quad$ via Eq.~\eqref{Eq:FCNs}}
    \State{Compute the attention matrix $\mathbf{A}$ via}
    \Statex{$\quad$ Eq.~\eqref{Eq:softmax-DoT}, and then obtain the transformed}
    \Statex{$\quad$ source features $\hat{\mathbf{F}}^s$ via Eq.~\eqref{eq:ot-map}}
    \Statex{$\quad$\% {\textit{Preparation of Local Structure Learning}}}
    \State{Calculate the local regularization terms $\mathcal{L}_{t}$}
    \Statex{$\quad$ and $\mathcal{L}_{s}$ via Eq.~\eqref{eq:lpp_t} and Eq.~\eqref{eq:lpp_s}}
    \State Calculate the entropy loss $\mathcal{L}_{ent}$ via Eq.~\eqref{eq:target-entropy-loss}
    \Statex{$\quad$\% {\textit{Update the overall objective}}}
    \State Update $\bm{\theta}$ by minimizing Eq.~\eqref{Eq:total-loss}
    \State Compute predictions as $\hat{y}_j^t = \arg\underset{k}\max\:\: \hat{p}_{jk}^t$
\EndFor
\end{algorithmic}
\end{algorithm}

DoT transforms the source features into target domain via the domain-level attention, and then it learns discriminative features with the help of local structure regularization. As shown in Eq.~\eqref{eq:barycentric-map}, the transformed source feature vectors $\hat{\mathbf{F}}^s$ and the target feature vectors $\mathbf{F}^t$ are in the shared subspace, the classifier trained on $\hat{\mathbf{F}}^s$ can be hopefully transferred to the target domain with a small generalization error. It implies that DoT implicitly achieves the domain alignment without introducing the explicit domain discrepancy optimization. Then, the empirical risk of the model can be directly estimated on the transformed source domain, i.e.,
\begin{equation}
\label{Eq:cross-entropy-loss}
\mathcal{L}_{tce}(\bm{\theta})\! =\! \frac{1}{n_s}\sum_{i = 1}^{n_s}l_{ce}(C(\hat{\bm{f}}_i^s),\bm{y}_i^s; \bm{\theta}_G, \bm{\theta}_{F_t}, \bm{\theta}_C),
\end{equation}
where $l_{ce}(\cdot, \cdot)$ is the cross-entropy function, and $\bm{\theta}$ represents all the parameters updated in the adaptive feature learning process. Here we use the subscript ``tce'' to stress that this loss term is computed by using the transformed source features $\hat{\mathbf{F}}^s$, rather than $\mathbf{F}^s$.

The overall objective function of DoT consists of four terms, and it is written as
\begin{equation}
\label{Eq:total-loss}
\begin{aligned}
\mathcal{L}(\bm{\theta}) = &\mathcal{L}_{tce}(\bm{\theta}) + \lambda_{ ent} \mathcal{L}_{ent}(\bm{\theta}_G, \bm{\theta}_{F_t}, \bm{\theta}_C)\\
&+ \lambda_{ s} \mathcal{L}_{s}(\bm{\theta}_G, \bm{\theta}_{F_s}, \bm{\theta}_{F_t})  + \lambda_{ t} \mathcal{L}_{t}(\bm{\theta}_G, \bm{\theta}_{F_t})  ,
\end{aligned}
\end{equation}
where $\lambda_{ent}$, $\lambda_{s}$ and $\lambda_{t} > 0$ are trade-off hyper-parameters. The locality preserving term $\mathcal{L}_s$ is built upon the domain-level attention mechanism to preserve discriminative structure in the source domain. The locality preserving term $\mathcal{L}_t$ and the entropy $\mathcal{L}_{ent}$ are used to explore intrinsic structure in the target domain, and then keep similar samples nearby in the latent subspace. Thus, each row of the attention matrix will have group-sparsity, and more relevant target feature vectors will make larger contributions to the corresponding outputs of the domain-transformer module. Overall, the locality preserving terms and domain-level attention mechanism are mutually beneficial, which ensures the learned features are both transferable and discriminative.

The main optimization procedure has been summarized in Algorithm~\ref{algorithm1}.
Note that we employ a pre-training strategy to initialize the two FCNs. Specifically, we pre-train the feature extractor $G(\cdot)$, source FCNs $F_s(\cdot)$, and classifier $C(\cdot)$ on the labeled source domain with cross-entropy loss. Then, we initialize the target FCNs $F_t(\cdot)$with the pre-trained $F_s(\cdot)$. Thus, the domain-level attention based on FCNs $F_{s/t}(\cdot)$ is expected to capture a more accurate sample correspondence. Enhanced by the local structure consistency learning phase, samples within the same underlying class but different domains will be drawn closer. Finally, it is expected that the discriminant structure can be learned and transferred from the source domain to the target domain.

\begin{table*}[t]
    \caption{Accuracies (\%) on ImageCLEF and Office-31 (ResNet-50).}
    \label{tab:results-clef-31}
    \begin{center}
    \renewcommand{\tabcolsep}{0.12pc}
    \renewcommand{\arraystretch}{1.0}
    \begin{tabular}{l|ccccccc|ccccccc}
    \toprule
    \multirow{2}{*}{\textbf{Method}} &\multicolumn{7}{c|}{\textbf{ImageCLEF}} &\multicolumn{7}{c}{\textbf{Office-31}}\\
                                      &  I$\rightarrow$P & P$\rightarrow$I & I$\rightarrow$C & C$\rightarrow$I & C$\rightarrow$P & P$\rightarrow$C & Mean & A$\rightarrow$W & D$\rightarrow$W & W$\rightarrow$D & A$\rightarrow$D & D$\rightarrow$A & W$\rightarrow$A & Mean \\
    \midrule
    Source-Only         & 74.8 & 83.9 & 91.5 & 78.0 & 65.5 & 91.2 & 80.7 & 68.4 & 96.7 & 99.3 & 68.9 & 62.5 & 60.7 & 76.1 \\
    DAN    & 74.5 & 82.2 & 92.8 & 86.3 & 69.2 & 89.8 & 82.5  & 80.5 & 97.1 & 99.6 & 78.6 & 63.6 & 62.8  & 80.4 \\		
    DANN       & 75.0 & 86.0 & 96.2 & 87.0 & 74.3 & 91.5 & 85.0 & 82.0 & 96.9 & 99.1 & 79.7 & 68.2 & 67.4 & 82.2 \\
    CDAN+E    & 77.7 & 90.7 & 97.7 & 91.3 & 74.2 & 94.3 & 87.7 & 94.1 & 98.6 & \textbf{100.0} & 92.9 & 71.0 & 69.3 & 87.7 \\
    KGOT   & 76.3 & 83.3& 93.5 & 87.5 & 74.8 & 89.0 & 84.1 & 75.3 & 96.2 & 98.4 & 80.3 & 65.2 & 63.5 & 79.8 \\
    DeepJDOT	 & 77.5 & 90.5  & 95.0 & 88.3 & 74.9 & 94.2 & 86.7 & 88.9 & 98.5 & 99.6 &	88.2 & 72.1 & 70.1 & 86.2 \\
    RWOT & 81.3 & 92.9 & 97.9& 92.7 & 79.1 & 96.5 & 90.0 & 95.1 & 99.5 & \textbf{100.0} & 94.5 & 77.5 & 77.9 & 90.8 \\
    ETD & 81.0 & 91.7 & 97.9 & 93.3 & 79.5 & 95.0 & 89.7 & 92.1 & \textbf{100.0} & \textbf{100.0} & 88.0 & 71.0 & 67.8 & 86.2 \\
    BuresNet   & 80.7 & 93.7 &	97.0 & 93.5 & 79.2 & 97.0 & 90.2&-&-&-&-&-&-&-\\
    SAFN & 78.0 & 91.7 & 96.2 & 91.1 & 77.0 & 94.7 & 88.1 & 88.8 & 98.4& 99.8 & 87.7& 69.8& 69.7 & 85.7 \\
    DSAN & 80.2 & 93.3& 97.2 & 93.8& \textbf{80.8}& 95.9 & 90.2 & 93.6 & 98.3 & \textbf{100.0} & 90.2 & 73.5 & 74.8 &88.4\\
    DSAN$\small{+}$CAFT &-&-&-&-&-&-&-& 92.0 & 98.0 & \textbf{100.0} & 88.0 & 75.0 & 75.0 & 88.0 \\
    TADA  &-&-&-&-&-&-&-  & 94.3 & 98.7 & 99.8 & 91.6 & 72.9 & 73.0 & 88.4 \\
    DMP & 80.7 & 92.5 & 97.2 & 90.5 & 77.7 & 96.2 & 89.1 & 93.0 & 99.0 & \textbf{100.0} & 91.0 & 71.4 & 70.2  & 87.4 \\
    ATM & 80.3 & 92.9 & \textbf{98.6} & 93.5 & 77.8 & 96.7 & 90.0 & \textbf{95.7}	&99.3	&\textbf{100.0}	&96.4&	74.1	&73.5 &	89.8  \\
    \textbf{DoT} & \textbf{81.4} & \textbf{95.5} & 97.8 & \textbf{95.3} & 80.4 & \textbf{98.3} & \textbf{91.5}  & 94.8 & 98.9 & 99.8 & \textbf{97.6} & \textbf{85.5} & \textbf{83.1} & \textbf{93.3}\\
    \bottomrule
    \end{tabular}
    \end{center}
  \end{table*}

\section{Experiments}\label{sect:experim}

In Section~\ref{sect:settings}, we introduce the datasets and experimental settings. In Section~\ref{sect:comparative results}, we show the experimental results of DoT on five datasets.

\subsection{Datasets and Experiment Settings}\label{sect:settings}
We conduct experiments on five UDA datasets.

\textbf{ImageCLEF} \citep{caputo2014imageclef} consists of 3 domains with 12 classes, i.e., \textit{Caltech} (\textbf{C}), \textit{ImageNet} (\textbf{I}) and \textit{Pascal} (\textbf{P}). Each domain has 600 images with 50 images per class.

\textbf{Office-31} \citep{saenko2010adapting} consists of 4,652 images from 3 domains with 31 classes, i.e., \textit{Amazon} (\textbf{A}), \textit{Webcam} (\textbf{W}) and \textit{DSLR} (\textbf{D}).

\textbf{Office-Home} \citep{venkateswara2017deep} consists of 15,500 images from 4 domains with 65 classes, i.e.,  \textit{Art} (\textbf{A}), \textit{Clipart} (\textbf{C}), \textit{Product} (\textbf{P}) and \textit{Real-World} (\textbf{R}).

\textbf{VisDA-2017} \citep{peng2017visda} consists of 280K images from 2 domains with 12 classes, i.e., \textit{synthetic} (\textbf{S}) and \textit{real-image} (\textbf{R}). Task \textbf{S}$\rightarrow$\textbf{R} will be used to algorithm evaluation.

\textbf{DomainNet.}~\citep{peng2019moment} is the largest domain adaptation dataset so far. It contains about 0.6 million images from 6 domains with 345 classes, i.e., \textit{Clipart} (\textbf{clp}), \textit{Infograph} (\textbf{inf}), \textit{Painting} (\textbf{pnt}), \textit{Quickdraw} (\textbf{qdr}), \textit{Real} (\textbf{rel}) and \textit{Sketch} (\textbf{skt}).

In experiments, we evaluate DoT by employing CNNs ResNet-50/101~\citep{he2016deep} or transformer ViT-B~\citep{dosovitskiy2021an} pre-trained on ImageNet~\citep{russakovsky2015imagenet} as the feature extractor $G(\cdot)$. Specifically, ViT-B represents ViT-Base with 16$\times$16 input patch size. Similar to transformer-based UDA methods~\citep{SSRT2022,xu2022cdtrans,SDAT2022}, the 768-dimensional class tokens of ViT-B are considered as latent features. The domain-specific projections $F_s(\cdot)$ and $F_t(\cdot)$ use the same architecture, which has three fully-connected layers for VisDA-2017 while two full-connected layers for other datasets. The classifier $C(\cdot)$ is a single fully-connected layer followed by the softmax activation function. We implement DoT with the PyTorch framework and train the network parameters with Adam Optimizer. More implementation details are presented in the supplementary material.

\begin{table*}[t]
    \caption{Accuracies (\%) on VisDA-2017 (ResNet-101).}
    \label{tab:results-visda}
    \begin{center}
    \renewcommand{\tabcolsep}{0.230pc}
    \begin{tabular}{l|cccccccccccc|c}
    \toprule
    \textbf{VisDA-2017} & plane & bcycl & bus & car & horse & knife & mcycl & person & plant & sktbrd & train & truck & Mean\\
    \midrule
    Source-Only  & 72.3 & 6.1 & 63.4 & 91.7 & 52.7 & 7.9 & 80.1 & 5.6 & 90.1 & 18.5 & 78.1 & 25.9 & 49.4\\
    DAN  & 68.1 & 15.4 & 76.5 & 87.0 & 71.1 & 48.9 & 82.3 & 51.5 & 88.7 & 33.2 & 88.9 & 42.2 & 62.8\\
    DANN  & 81.9 & 77.7 & 82.8 & 44.3 & 81.2 & 29.5 & 65.1 & 28.6 & 51.9 & 54.6 & 82.8 & 7.8 & 57.4\\
    CDAN & 85.2 & 66.9 & 83.0 & 50.8 & 84.2 & 74.9 & 88.1 & 74.5 & 83.4 & 76.0 & 81.9 & 38.0 & 73.9\\
    JAN  & 75.7 & 18.7 & 82.3 & 86.3 & 70.2 & 56.9 & 80.5 & 53.8 & 92.5 & 32.2 & 84.5 & 54.5 & 65.7\\
    DeepJDOT	 & 85.4 & 73.4 & 77.3 & 87.3 & 84.1 & 64.7 & 91.5 & 79.3 & 91.9 & 44.4 & 88.5 & 61.8 & 77.4\\
    RWOT & 95.1& 80.3& \textbf{83.7}& \textbf{90.0}& 92.4& 68.0& 92.5& 82.2& 87.9& 78.4& 90.4& \textbf{68.2}& 84.0 \\
    BuresNet & 92.5 & 56.5 & \textbf{83.7} & 82.4 & 93.5 & 85.7 & 93.5 & 71.1 & 93.4 & 82.7 & 83.9 & 18.4 & 78.1 \\
    TPN & 93.7 & \textbf{85.1} & 69.2 & 81.6 & 93.5 & 61.9 & 89.3 & 81.4 & \textbf{93.5} & 81.6 & 84.5 & 49.9 & 80.4\\
    DSAN & 90.9 & 66.9 & 75.7 & 62.4 & 88.9 & 77.0 & \textbf{93.7} & 75.1 & 92.8 & 67.6 & 89.1 & 39.4 & 75.1\\
    DSAN$\small{+}$CAFT& 91.5 & 70.1 & 74.9 & 55.1 & 90.2 & 71.0 & 86.9 & 76.2 & 92.4 & 78.1 & 91.3 & 45.3 & 76.9\\
    DMP & 92.1 & 75.0 & 78.9 & 75.5 & 91.2 & 81.9 & 89.0 & 77.2 & 93.3 & 77.4 & 84.8 & 35.1 & 79.3\\
    ATM& 93.9 & 65.1 & \textbf{83.7} & 72.1 & 92.2 & 92.5 & 92.4 & 79.7 & 86.1 & 47.8 & 86.1 & 22.0 & 76.1 \\
    CDAN+SDAT&94.8 & 77.1&82.8& 60.9& 92.3& \textbf{95.2}&91.7& 79.9& 89.9& \textbf{91.2}& 88.5& 41.2& 82.1\\
    \textbf{DoT} & \textbf{96.1} & 81.4 & 71.1 & 78.6 & \textbf{94.2} & 91.2 & 88.8 & \textbf{85.4} & 92.3 & 90.6 & \textbf{92.7} & 56.4& \textbf{84.9}\\
    \bottomrule
    \end{tabular}
    \end{center}
  \end{table*}

\begin{table*}
    \caption{Accuracies (\%) on Office-Home (ResNet-50).}
    \label{tab:results-home}
    \begin{center}
    \renewcommand{\tabcolsep}{0.2pc} 
    \begin{tabular}{l|cccccccccccc|c}
    \toprule
\textbf{Office-Home} & A$\to$C & A$\to$P & A$\to$R & C$\to$A & C$\to$P & C$\to$R & P$\to$A & P$\to$C & P$\to$R & R$\to$A & R$\to$C& R$\to$P & Mean \\
    \midrule
    Source-Only & 34.9 & 50.0 & 58.0 & 37.4 & 41.9 & 46.2 & 38.5 & 31.2 & 60.4 & 53.9 & 41.2 & 59.9 & 46.1 \\
    DAN & 43.6 & 57.0 & 67.9 & 45.8 & 56.5 & 60.4 & 44.0 & 43.6 & 67.7 & 63.1 & 51.5 & 74.3 & 56.3 \\				
    DANN & 45.6 & 59.3 & 70.1 & 47.0 & 58.5 & 60.9 & 46.1 & 43.7 & 68.5 & 63.2 & 51.8 & 76.8 & 57.6 \\
    CDAN+E  & 50.7 & 70.6 & 76.0 & 57.6 & 70.0 & 70.0 & 57.4 & 50.9 & 77.3 & 70.9 & 56.7 & 81.6 & 65.8 \\
    KGOT&36.2&59.4&65.0&48.6&56.5&60.2&52.1&37.8&67.1&59.0&41.9&72.0&54.7\\
    DeepJDOT	& 48.2 & 69.2 & 74.5 & 58.5 & 69.1 & 71.1 & 56.3 & 46.0 & 76.5 & 68.0 & 52.7 & 80.9 & 64.3\\
    RWOT& 55.2 & 72.5 & 78.0 & 63.5 & 72.5 & 75.1 &  60.2 & 48.5 & 78.9 & 69.8 & 54.8 & 82.5 & 67.6\\
    ETD  & 51.3 & 71.9 & \textbf{85.7} & 57.6 & 69.2 & 73.7 &  57.8 & 51.2 & 79.3 & 70.2 & 57.5 & 82.1 & 67.3\\
    BuresNet  & 54.7 & 74.4 & 77.1 & 63.7 & 72.2 & 71.8 &	64.1 & 51.7 & 78.4 & 73.1 & 58.0 & 82.4 & 68.5\\
    DSAN & 54.4 & 70.8 & 75.4 & 60.4 & 67.8 & 68.0 & 62.6 & \textbf{55.9} & 78.5 & 73.8 & 60.6 & 83.1 & 67.6 \\
    DSAN$\small{+}$CAFT & 55.2 & 69.8 & 75.0 & 60.0 & 72.0 & 71.0 & 63.3 & 57.3 & 79.1 & 74.1 & 60.6 & 83.0 & 68.4\\
    TADA& 53.1 &  72.3  & 77.2 &  59.1 &  71.2 &  72.1  & 59.7 &  53.1 &  78.4 &  72.4 &  60.0 &  82.9 &  67.6 \\
    DMP  & 52.3 & 73.0 & 77.3 & 64.3 & 72.0 & 71.8 & 63.6 & 52.7 & 78.5 & 72.0 & 57.7 & 81.6 & 68.1 \\
    ATM & 52.4 & 72.6 & 80.2 & 61.1 & 72.0 & 72.6 & 59.5 & 52.0 & 79.1 & 73.3 & 58.9 & 83.4 & 67.9  \\
    CDAN+SDAT&\textbf{56.0} & 72.2& 78.6& 62.5 &73.2 & 71.8 & 62.1 & \textbf{55.9} & 80.3 & \textbf{75.0} & \textbf{61.4} & \textbf{84.5} & \textbf{69.5}\\
    \textbf{DoT} & 52.9 & 75.2 & 80.2 & \textbf{66.0} & \textbf{77.7} & \textbf{78.9} & 66.0 & 50.5 & \textbf{82.0} & 68.4 & 51.2 & 82.0 & \textbf{69.3}\\
    \bottomrule
    \end{tabular}
    \end{center}
  \end{table*}

\subsection{Comparative Results}\label{sect:comparative results}

\begin{table*}[t]
\caption{Accuracies (\%) on DomainNet (ResNet-101). The column-wise means represent the source domains, while the row-wise means represent the target domains. The symbol $^\dag$ is used to emphasize that the mean classification accuracy is computed across 20 tasks. Under the same calculation rule as CDAN+SDAT, the mean accuracy of DoT is \textbf{49.0}\%.}
\label{tab:results on DomainNet}
\centering
\renewcommand{\tabcolsep}{0.18pc}
\begin{tabular}{c|ccccccc|c|ccccccc}
\toprule
    ResNet-101  & clp & inf & pnt & qdr & rel & skt & Avg. & CDAN & clp & inf & pnt & qdr & rel & skt & Avg. \\ \hline
    clp & - & 19.3 & 37.5 & 11.1 & 52.2 & 41.0 & 32.2 & clp & - & 20.4 & 36.6 & 9.0 & 50.7 & 42.3 & 31.8 \\
    inf & 30.2 & - & 31.2 & 3.6 & 44.0 & 27.9 & 27.4 & inf &  27.5 & - & 25.7 & 1.8 & 34.7 & 20.1 & 22.0 \\
    pnt & 39.6 & 18.7 & - & 4.9 & 54.5 & 36.3 & 30.8 & pnt &  42.6 & 20.0 & - & 2.5 & 55.6 & 38.5 & 31.8 \\
    qdr & 7.0 & 0.9 & 1.4 & - & 4.1 & 8.3 & 4.3 & qdr &  21.0 & 4.5 & 8.1 & - & 14.3 & 15.7 & 12.7 \\
    rel & 48.4 & 22.2 & 49.4 & 6.4 & - & 38.8 & 33.0 & rel & 51.9 & 23.3 & 50.4 & 5.4 & - & 41.4 & 34.5  \\
        skt & 46.9 & 15.4 & 37.0 & 10.9 & 47.0 & - & 31.4 & skt & 50.8 & 20.3 & 43.0 & 2.9 & 50.8 & - & 33.6 \\
        Avg. & 34.4 & 15.3 & 31.3 & 7.4 & 40.4 & 30.5 & 26.6 & Avg. & 38.8 & 17.7 & 32.8 & 4.3 & 41.2 & 31.6 & 27.7 \\ \hline
    MDD & clp & inf & pnt & qdr & rel & skt & Avg. & MDD+SCDA & clp & inf & pnt & qdr & rel & skt & Avg. \\ \hline
    clp & -& 20.5 & 40.7 & 6.2 & 52.5 & 42.1 & 32.4 & clp & - & 20.4 & 43.3 & 15.2 & 59.3 & 46.5 & 36.9 \\
    inf & 33.0 & - & 33.8 & 2.6 & 46.2 & 24.5 & 28 & inf & 32.7 & - & 34.5 & 6.3 & 47.6 & 29.2 & 30.1 \\
    pnt & 43.7 & 20.4 & - & 2.8 & 51.2 & 41.7 & 32.0 & pnt & 46.4 & 19.9 & - & 8.1 & 58.8 & 42.9 & 35.2 \\
    qdr  & 18.4 & 3.0 & 8.1 & - & 12.9 & 11.8 & 10.8 & qdr & 31.1 & 6.6 & 18.0 & - & 28.8 & 22.0 & 21.3 \\
    rel & 52.8 & 21.6 & 47.8 & 4.2 & - & 41.2 & 33.5 & rel & 55.5 & 23.7 & 52.9 & 9.5 & - & 45.2 & 37.4 \\
    skt & 54.3 & 17.5 & 43.1 & 5.7 & 54.2 & - & 35.0 & skt & 55.8 & 20.1 & 46.5 & 15.0 & 56.7 & - & 38.8 \\
    Avg. & 40.4 & 16.6 & 34.7 & 4.3 & 43.4 & 32.3 & 28.6 & Avg. & 44.3 & 18.1 & 39.0 & 10.8 & 50.2 & 37.2 & 33.3 \\ \hline
 CDAN+SDAT & clp & inf & pnt & qdr & rel & skt & Avg. & DoT & clp & inf & pnt & qdr & rel & skt & Avg. \\\hline
 clp & - & 22.0 & 41.5 & - & 57.5 & 47.2 & 42.1 & clp & - & 19.7 & 44.2 & 19.0 & 63.6 & 47.5 & 38.8 \\
inf & 33.9 & -& 30.3& -  & 48.1 & 27.9 & 35.0 & inf & 57.1 & - & 43.4 & 10.7 & 62.6 & 41.3 & 43.0 \\
pnt & 47.5 & 20.7 & - & - &58.0  & 41.8 & 42.0 & pnt & 59.4 & 25.0 & - & 15.4 & 66.0 & 48.0 & 42.7 \\
qdr &- & - & - & - & - & - & - & qdr & 29.9 & 3.5 & 9.1 & - & 17.9 & 20.7 & 16.2 \\
rel & 56.7 & 25.1 & 53.6 & - & - & 43.9 & 44.8 & rel & 63.0 & 26.6 & 55.7 & 23.4 & - & 48.6 & 43.4 \\
skt & 58.7 & 21.8 & 48.1 & - & 57.1 & - & 46.4 & skt & 64.4 & 23.0 & 52.1 & 27.5 & 67.9 & - & 47.0 \\
 Avg. & 49.2 & 22.4 & 43.4& -  & 55.2 & 40.2 & 42.1$^\dag$ & Avg. & 54.8 & 19.6 & 40.9 & 19.2 & 55.6 & 41.2 & \textbf{38.5}\\
\bottomrule
\end{tabular}
\end{table*}

\textbf{CNNs-based Results.} The compared methods can be roughly categorized into three groups. (1): Moment alignment methods, e.g., DAN~\citep{long2019transferable}, JAN~\citep{long-JAN}, TPN~\citep{Pan_2019_CVPR}, DSAN~\citep{zhu2021deep}, DSAN$\small{+}$CAFT~\citep{kumar2023improving}, and BuresNet~\citep{BuresNet2023}. (2): Adversarial methods, e.g., DANN~\citep{ganin2016domain}, CDAN~\citep{long2018conditional}, CDAN+E~\citep{long2018conditional}, CDAN+SDAT~\citep{SDAT2022}, TADA~\citep{wang2019transferable}, and SCDA~\citep{SCDA2021}. (3): Manifold and OT-based methods, e.g., KGOT~\citep{zhang2019optimal}, DeepJDOT~\citep{Damodaran_2018_ECCV}, ETD~\citep{li2020enhanced}, RWOT~\citep{RWOT2020} and DMP~\citep{luo2022unsupervised}. MDD~\citep{MDD2019}, SAFN~\citep{xu2019larger} and ATM~\citep{ATM} are also used for comparison.

The results of ImageCLEF are shown in the left of Table~\ref{tab:results-clef-31}. For the OT-based methods, DeepJDOT, RWOT, and ETD improve the mean accuracy of KGOT by considering the label information. It indicates that discriminative feature learning can enhance the classification performance of adaptation models. The mean accuracy of DoT is 1.8\% and 4.8\% higher than that of ETD and DeepJDOT, respectively. Such results indicate that building intrinsic sample interactions across domains is effective in learning discriminative features. Besides, the mean accuracy of DoT is at least 1.5\% higher than that of the manifold-based method DMP and metric-based method ATM. The accuracy of BuresNet and DSAN both achieve 90.2\%, which is 1.3\% lower than that of DoT. DMP, ATM, BuresNet, and DSAN attempt to learn discriminative and transferable features by aligning class-relevant clusters across domains. DoT achieves the highest mean accuracy, which indicates its effectiveness in achieving locality consistency.

The results of Office-31 are shown in the right of Table~\ref{tab:results-clef-31}. Extended from DANN, TADA highlights transferable image regions by incorporating region-wise domain discriminator. Thus, TADA improves the mean accuracy of DANN by 6.2\%. Different from TADA, DoT learns transferable features by mapping the source samples into the target domain via domain-level attention, and it obtains the highest accuracy 93.3\%. Most comparison methods, e.g., CDAN+E, DSAN, DSAN+CAFT, RWOT, BuresNet, and ATM employ target pseudo-labels in different ways. Compared with these methods, DoT also achieves superior mean accuracy. Note that D$\rightarrow$A and W$\rightarrow$A are harder tasks since the domain discrepancy (real-to-virtual) puts forward greater demands on the generalization ability of features. We notice that DoT achieves the best accuracies 75.8\% and 75.0\% on these two tasks, respectively. These results validate that DoT is effective in reducing domain discrepancy.

\begin{table*}[t!]
\caption{Accuracies (\%) on Office-Home and VisDA-2017 (ViT-B). $^{*}$CDTrans and WinTR uses DeiT backbone. ``-B/S'' indicates the backbone is -Base/Small, respectively.}
\label{tab:results on Officehome ViT}
\begin{center}
\renewcommand{\tabcolsep}{0.25pc}
\begin{tabular}{l|cccccccccccc|c}
\toprule
\textbf{VisDA-2017} & plane & bcycl & bus & car & house & knife & mcycl & person & plant & sktbrd & train & truck & Mean \\
\midrule
ViT-B &98.2 & 74.4 & 81.2 & 68.2 & 95.8 & 67.5 & 96.2 & 22.1 & 73.4 & 85.4 & 95.0 & 17.6 & 72.9\\
CDTrans$^{*}$& 97.1 & 90.5 & 82.4 & 77.5 & 96.6 & 96.1 & 93.6 & \textbf{88.6} & \textbf{97.9} & 86.9 & 90.3 & \textbf{62.8} & 88.4 \\
WinTR$^{*}$ & 98.7 & 91.2 & \textbf{93.0} & \textbf{91.9} & 98.1 & 96.1 & 94.0 & 72.7 & 97.0 & 95.5 & 95.3 & 57.9 & 90.1 \\
TVT & 92.9 &85.6& 77.5 &60.5& 93.6& 98.2 &89.4& 76.4 &93.6 &92.0 &91.7 &55.7 &83.9\\
SSRT & 98.9 & 87.6 & 89.1 & 84.8 & 98.3 & 98.7 & \textbf{96.3} & 81.1 & 94.9 & 97.9 & 94.5 & 43.1 & 88.8 \\
CDAN+SDAT&96.3 & 80.7 & 74.5 & 65.4 & 95.8 & 99.5 & 92.0 & 83.7 & 93.6 & 88.9 & 85.8 & 57.2 & 84.5\\
PMTrans & 98.9 & 93.7 & 84.5 & 73.3 & \textbf{99.0} & 98.0 & 96.2 & 67.8 & 94.2 & 98.4 & 96.6 & 49.0 & 87.5\\
\textbf{DoT} & \textbf{99.1} &\textbf{ 92.5} & 89.6 & 83.0 & 98.3 & \textbf{99.8} & 95.5 & 81.7 & 91.6 & \textbf{98.7} & \textbf{96.8} & 62.1 & \textbf{90.7}\\
\bottomrule
\end{tabular}
\end{center}
\vskip 0.02in
\begin{center}
\renewcommand{\tabcolsep}{0.22pc}
\begin{tabular}{l|cccccccccccc|c}
\toprule
\textbf{Office-Home} & A$\to$C & A$\to$P & A$\to$R & C$\to$A & C$\to$P & C$\to$R & P$\to$A & P$\to$C & P$\to$R & R$\to$A & R$\to$C& R$\to$P & Mean \\
\midrule
ViT-B & 60.4 & 83.4 & 86.8 & 74.6 & 83.1 & 84.2 & 73.5 & 60.0 & 86.6 & 76.2 & 60.8 & 87.8 & 76.4 \\
CDTrans$^{*}$& 68.8 & 85.0 & 86.9 & 81.5 & 87.1 & 87.3 & 79.6 & 63.3 & 88.2 & 82.0 & 66.0 & 90.6 & 80.5\\
WinTR-S$^{*}$ & 65.3 & 84.1 & 85.0 & 76.8 & 84.5 & 84.4 & 73.4 & 60.0 & 85.7 & 77.2 & 63.1 & 86.8 & 77.2\\
TVT& 74.9& 86.8 & 89.5 & 82.8 & 88.0 & 88.3 & 79.8 & 71.9 & 90.1 & 85.4 & 74.6 & 90.6 & 83.6 \\
CDAN+SDAT&69.1&86.6& 88.9& 81.9 &86.2 & 88.0 & 81.0 & 66.7 & 89.7 & 86.2 & 72.1 & 91.9 & 82.4\\
SSRT & 75.2 &  89.0 &  91.1 &  85.1 &  88.3 &  90.0 &  85.0 &  74.2 &  91.3 &  85.7 &  78.6 &  91.8 &  85.4 \\
PMTrans & \textbf{81.2} &\textbf{91.6} & \textbf{92.4} & \textbf{88.9} &\textbf{91.6} & \textbf{93.0} & \textbf{88.5} & \textbf{80.0} & \textbf{93.4} & \textbf{89.5} & \textbf{82.4} &\textbf{94.5} &\textbf{88.9} \\
\textbf{DoT} & 72.9 & 89.8 & 90.3 & 81.8 & 89.6 & 90.1 & 81.2 & 70.6 & 92.4 & 82.9 & 72.2 & 90.8 & \textbf{83.7} \\
\bottomrule
\end{tabular}
\end{center}
\end{table*}

The results of Office-Home are shown in Table~\ref{tab:results-home}. Office-Home is a challenging dataset since it has 65 classes. CDAN+E extends DANN by introducing label information into the adversarial training process and improves the mean accuracy to 65.8\%. Besides, DeepJDOT, RWOT, and ETD outperform KGOT, which demonstrates that discriminative features are vital for enhancing the classification performance on datasets with many classes. Note that the accuracy of DoT is 1.7\% higher than that of DSAN, which reduces the domain discrepancy in a class-wise manner. Though BuresNet explicitly characterizes the class-conditional distribution discrepancy in the kernel space, its accuracy is lower than that of DoT by 0.8\%.
The mean accuracy of CDAN+SDAT is 0.2\% higher than DoT due to the improved adversarial training. The superiority of DoT over the remaining methods can be attributed to the exploration of locality consistency. Besides, unlike most methods, DoT learns transferable and discriminative features without employing pseudo-labels.

\begin{table*}[t]
\caption{Accuracies (\%) on DomainNet (ViT-B).}
\label{tab:results on DomainNet(ViT)}
\centering
\renewcommand{\tabcolsep}{0.28pc}
\begin{tabular}{c|ccccccc|c|ccccccc}
\toprule
ViT-B  & clp & inf & pnt & qdr & rel & skt & Avg. & CDtrans$^{*}$ & clp & inf & pnt & qdr & rel & skt & Avg. \\ \hline
clp  & - & 27.2 & 53.1 & 13.2 & 71.2 & 53.3 & 43.6 &  clp  &-& 27.9 & 57.6 & 27.9 & 73.0 & 58.8 & 49.0 \\
inf & 51.4 &-& 49.3 & 4.0 & 66.3 & 41.1 & 42.4 & inf & 58.6 & - & 53.4 & 9.6 & 71.1 & 47.6 & 48.1 \\
pnt & 53.1 & 25.6 &-& 4.8 & 70.0 & 41.8 & 39.1 & pnt & 60.7 & 24.0 & - & 13.0 & 69.8 & 49.6 & 43.4 \\
qdr & 30.5 & 4.5 & 16.0 &-& 27.0 & 19.3 & 19.5 & qdr & 2.9 & 0.4 & 0.3 & - & 0.7 & 4.7 & 1.8 \\
rel& 58.4 & 29.0 & 60.0 & 6.0 &-& 45.8 & 39.9 & rel & 49.3 & 18.7 & 47.8 & 9.4 & - & 33.5 & 31.7\\
skt & 63.9 & 23.8 & 52.3 & 14.4 & 67.4 &-& 44.4 & skt & 66.8 & 23.7 & 54.6 & 27.5 & 68.0 & - & 48.1 \\
Avg. & 51.5 & 22.0 & 46.1 & 8.5 & 60.4 & 40.3 & 38.1 & Avg. & 47.7 & 18.9 & 42.7 & 17.5 & 56.5 & 38.8 & 37.0 \\
\hline
SSRT  & clp & inf & pnt & qdr & \text { rel } & skt & Avg. & AMPT  & clp & inf & pnt & qdr & rel  & skt & Avg. \\
\hline
clp & - & 33.8 & 60.2 & 19.4 & 75.8 & 59.8 & 49.8 &clp &-& 29.5 & 57.2 & 29.0 & 73.0 & 58.6 & 49.5 \\
inf & 55.5 &-& 54.0 & 9.0 & 68.2 & 44.7 & 46.3 & inf & 56.2 &-& 54.6 & 13.1 & 69.9 & 48.8 & 48.5 \\
pnt & 61.7 & 28.5 &-& 8.4 & 71.4 & 55.2 & 45.0 & pnt & 62.2 & 27.8 &-& 16.2 & 71.9 & 53.9 & 46.4 \\
qdr& 42.5 & 8.8 & 24.2 &-& 37.6 & 33.6 & 29.3 & qdr& 40.8 & 10.2 & 29.4 &-& 44.8 & 28.6 & 30.8 \\
rel & 69.9 & 37.1 & 66.0 & 10.1 &-& 58.9 & 48.4 &  rel & 66.3 & 31.3 & 61.2 & 17.2 &-& 52.7 & 45.7 \\
skt & 70.6 & 32.8 & 62.2 & 21.7 & 73.2 &-& 52.1 & skt & 69.1 & 29.7 & 58.4 & 28.0 & 71.7 &-& 51.4 \\
Avg. & 60.0 & 28.2 & 53.3 & 13.7 & 65.3 & 50.4 & 45.2 & Avg. & 58.9 & 25.7 & 52.2 & 20.7 & 66.3 & 48.5 & 45.4 \\
\hline
PMTrans  & clp& inf  & pnt & qdr& rel  & skt & Mean & DoT & clp & inf & pnt &qdr & rel  & skt & Mean  \\
\hline
clp &-& 34.2 & 62.7 & 32.5 & 79.3 & 63.7 & 54.5 & clp  & - & 31.8 &	61.3 & 28.3 & 79.6 & 60.4 &	52.3  \\
inf & 67.4 &-& 61.1 & 22.2 & 78.0 & 57.6 & 57.3 & inf & 71.7 & - &  60.1 & 19.4 & 79.8 & 57.7 &	57.7  \\
pnt & 69.7 & 33.5 &-& 23.9 & 79.8 & 61.2 & 53.6 & pnt & 73.2 & 35.1 & - & 23.6 & 79.2 &	59.7 & 54.2  \\
qdr & 54.6 & 17.4 & 38.9 &-& 49.5 & 41.0 & 40.3 & qdr& 57.5 & 16.7 & 43.4 & - & 61.4 & 42.4 & 44.3  \\
 rel  & 74.1 & 35.3 & 70.0 & 25.4 &-& 61.1 & 53.2 & rel & 71.1 & 33.0 & 63.2 & 34.1 & - & 56.4 & 51.6  \\
 skt & 73.8 & 33.0 & 62.6 & 30.9 & 77.5 &-& 55.6 & skt & 72.6 & 30.5 & 61.9 & 36.6 & 78.3 & - & 56.0 \\
Mean & 67.9 & 30.7 & 59.1 & 27.0 & 72.8 & 56.9 & 52.4 &  Mean & 69.2 & 29.4 & 58.0 & 28.4 & 75.7 & 55.3 & \textbf{52.7}  \\
\bottomrule
\end{tabular}
\end{table*}
\begin{table*}[t]
\caption{Accuracies (\%) on Office-31 (ViT-B).}
\label{tab: results on Office31 ViT}
\begin{center}
\renewcommand{\tabcolsep}{0.65pc}
\begin{tabular}{l|cccccc|c}
\toprule
\textbf{Office-31} & A$\rightarrow$W& D$\rightarrow$W& W$\rightarrow$D& A$\rightarrow$D& D$\rightarrow$A& W$\rightarrow$A & Mean \\
\midrule
ViT-B~\citep{dosovitskiy2021an} & 88.7 & 99.0 & \textbf{100.0} & 86.9 & 78.9 & 79.7 & 88.9\\
CDTrans$^{*}$~\citep{xu2022cdtrans} & 96.7 & 99.0 & \textbf{100.0} & \textbf{97.0} & 81.1 & 81.9 & 92.6 \\
TVT~\citep{yang2023tvt} & 96.4 & 99.4 & \textbf{100.0 }& 96.4 & 84.9 & 86.1 & 93.9 \\
SSRT~\citep{SSRT2022} & 97.7 & 99.2 & \textbf{100.0} & 98.6 & 83.5& 82.2 & 93.5 \\
AMPT~\citep{AMPT2023} &97.1 & 98.9 & \textbf{100.0} & 97.4 & 81.2 & 82.0 & 92.8\\
PMTrans~\citep{PMTrans2023} & \textbf{99.1} & \textbf{99.6} & \textbf{100.0} & \textbf{99.4} & \textbf{85.7} & 86.3 & \textbf{95.0} \\
\textbf{DoT} & 96.6 & 99.4 & \textbf{100.0} & 96.7 & 85.1 & \textbf{86.8} & 94.1 \\
\bottomrule
\end{tabular}
\end{center}
\end{table*}

VisDA-2017 is a challenging dataset due to class imbalance and large virtual-to-real domain discrepancy. Table~\ref{tab:results-visda} shows the accuracies of each class and mean accuracy. It can seen that \textit{truck} is a common barrier for all the methods. DeepJDOT performs best on the \textit{truck} class while its mean accuracy is only 77.4\%. Perhaps the class imbalance problem disturbs its optimal transport map in the joint space of features and class labels. DoT achieves rank-2 accuracy in the~\textit{truck} class and the highest mean accuracy. It indicates that achieving locality consistency is reliable in dealing with the class imbalance problem. TPN simultaneously reduces a global and class-wise domain discrepancy, and DMP learns the discriminative structure of target domains through a fine-grained manifold metric alignment. We see that DoT outperforms all the remaining methods with the highest mean accuracy 84.9\%, which is 4.5\% and 5.6\% higher than TPN and DMP, respectively. These results indicate that DoT is helpful in reducing significant domain discrepance.

Table~\ref{tab:results on DomainNet} shows the classification accuracy of 30 transfer tasks on DomainNet. The mean accuracy of ResNet-101 is only 26.6\% since DomainNet is a large dataset with high difficulty and diversity. SCDA+MDD improves the mean accuracy of MDD by proposing a pair-wise adversarial domain alignment loss of prediction distributions. It can be seen that DoT outperforms SCDA+MDD by 5.2\%. SDAT improves the performance of CDAN by enhancing the smoothness w.r.t. task loss in the adversarial training process. We notice that the mean accuracy of CDAN+SDAT is based on 25 tasks. Under the same calculation rule, the mean accuracy of DoT is 49.0\%, which outperforms CDAN+SDAT by 6.9\%. These results demonstrate that our DoT has a strong generalization ability for the challenging UDA problem.

\textbf{Transformer-based Results}. To verify the generalization of our proposal, we further evaluate DoT by employing ViT-B as the backbone. The comparison methods include CDTrans~\citep{xu2022cdtrans}, WinTR~\citep{ma2021WinTR}, SSRT~\citep{SSRT2022}, CDAN+SDAT~\citep{SDAT2022}, TVT~\citep{yang2023tvt}, AMPT~\citep{AMPT2023}, and PMTrans~\citep{PMTrans2023}.

\begin{table*}[t]\small
\caption{Component analysis of the proposed DoT model on different datasets (\%).}
 \label{tab:ablation_experiment}
 \renewcommand{\tabcolsep}{1.2pc}
 \begin{center}
 \begin{tabular}{cccc|cc|c|cc}
 \toprule
 \multicolumn{4}{c|}{}  &  \multicolumn{2}{c|} {ImageCLEF} & VisDA-2017 &\multicolumn{2}{c} {Office-31}\\
 $ \mathcal{A}tt.$        & $\mathcal{L}_s$       & $\mathcal{L}_t$    & $\mathcal{L}_{ent}$     & C$\rightarrow$I & I$\rightarrow$C & S$\rightarrow$R& A$\rightarrow$W & W$\rightarrow$A\\
 \midrule
  DA &  &  &                          & 92.8 & 96.0 & 81.4 & 93.6 & 82.2 \\
  DA & $\checkmark$ & &               & 94.3 & 97.3 & 82.5 & 93.5 & 82.7 \\
  DA & & $\checkmark$ & $\checkmark$  & 94.7 & 97.3 &  83.3 & 93.8 & 82.7 \\
  DA & $\checkmark$&$\checkmark$&$\checkmark$& \textbf{95.3} &\textbf{97.8}  & \textbf{84.9}& \textbf{94.8} & \textbf{83.1} \\
  PA & $\checkmark$ & $\checkmark$ & $\checkmark$ & 93.7 & 96.8 & 68.8 & 85.5 & 69.4   \\
 \bottomrule
 \end{tabular}
 \end{center}
\end{table*}

Table~\ref{tab: results on Office31 ViT} reports the results on Office-31. Compared with the ViT-B baseline, DoT achieves a much higher mean accuracy, which shows that building the interactions between cross-domain samples ensures further improvement in addressing the domain shift. With the help of domain-level and patch-level adversarial learning, TVT improves the transferability of the attention mechanism by assigning weights to different patches. Though the framework is simple, DoT achieves a better mean accuracy than TVT, which indicates the effectiveness of our DoT.

Table~\ref{tab:results on Officehome ViT} reports the results on Office-Home and VisDA-2017. The mean accuracy of DoT is 17.8\% higher than that of ViT-B on VisDA-2017, which indicates the effectiveness of DoT in reducing domain discrepancy. Besides, WinTR is based on DeiT-S while other methods are based on DeiT/ViT-B. DeiT-S is a smaller model than DeiT/ViT-B. It might be the reason that other methods achieve much better performance than WinTR on Office-Home. Unlike the CNNs-based results, DoT outperforms CDAN+SDAT on Office-Home, which shows the promising potential of ViT.
Since Office-Home has 65 categories, the self-training strategy in SSRT and the intermediate domain in PMtrans play a vital role in boosting performance. It might be the reason that SSRT and PMtrans perform better than DoT on Office-Home. It is worth noting that the mean accuracy of DoT surpasses all the methods on VisDA-2017, including SSRT and PMtrans. These results show that DoT has an advantage in dealing with large-scale datasets.

\begin{figure}[t]
    \centering
    \includegraphics[scale=0.9]{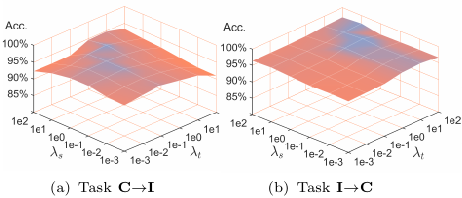}
    \caption{Parameter analysis curves of $\lambda_{s}$ and $\lambda_{t}$ on ImageCLEF.}
    \label{fig:hyper-param}
\end{figure}

Table~\ref{tab:results on DomainNet(ViT)} report the results on DomainNet. It can be seen that DoT achieves the highest mean accuracy 52.7\%. Besides, DoT outperforms the ViT-B baseline by 14.6\%, which further validates that DoT is effective in dealing with the domain discrepancy in UDA. With qdr being the more challenging target domain, while the others serve as source domains, our DoT achieves a mean accuracy of 28.4\%. In comparison, SSRT and CDTrans only achieve mean accuracies of 13.7\% and 17.5\%, respectively. Further, we can notice that the mean accuracy of DoT outperforms PMTrans and SSRT on the two challenging datasets, i.e., VisDA-2017 and DomainNet. These results demonstrate the effectiveness of DoT in handling these difficult UDA scenarios.

Compared with the CNNs-based results in Table~\ref{tab:results-clef-31}-\ref{tab:results-home}, the transformer-based results are much better, especially on the challenging datasets Office-Home and VisDA-2017. This is attributed to the strong transferable features obtained by transformers. DoT based on ResNet or ViT always achieves the best mean accuracies on Office-31, Office-Home and VisDA-2017. Thus, we conclude that DoT can promote classification performance under different backbones.

\begin{figure*}[t]
    \centering
    \includegraphics[scale=0.89]{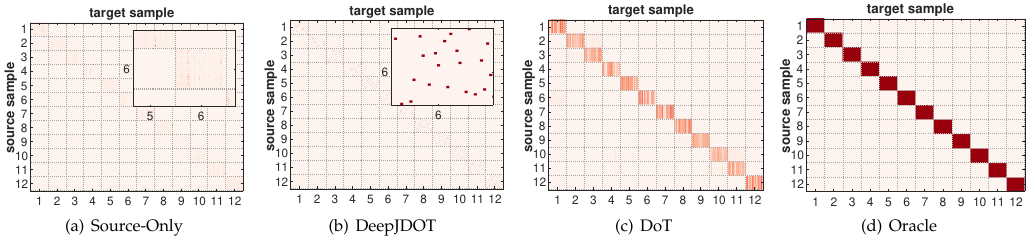}
    \caption{Heat-maps of attention matrices on ImageCLEF task I$\rightarrow$C. Darker colors represent larger values. Best viewed in color.}
    \label{fig:heatmap_imageclef-i2c}
\end{figure*}

\section{Method Analysis}\label{sect:results-analysis}
In this section, for the generality of analysis, we conduct method analysis on DoT with ResNet-50/101 as the backbone. Further analysis can be found in the supplementary material.

\subsection{Ablation study}
\subsubsection{Component Analysis}
The objective function of DoT can be roughly decomposed into three components, i.e., the cross-entropy term $\mathcal{L}_{tce}$ based on the domain-level attention, the supervised local structure learning term $\mathcal{L}_s$, and the unsupervised local structure learning terms $\mathcal{L}_t$ and $\mathcal{L}_{ent}$. To explore the impacts of these parts, we conduct an ablation study and show the results in Table~\ref{tab:ablation_experiment}. Note that $\mathcal{L}_{tce}$ has been replaced by $\mathcal{A}tt.$, which is short for Attention, to distinguish the patch-level attention (PA) in ViT and our domain-level attention (DA).

The accuracies shown in the first row indicate that the domain-level attention, which is implemented via the cross-entropy term $\mathcal{L}_{tce}$, nearly achieves SOTA performance. These results validate the theoretical connection, as shown in Section \ref{subsec:connection_OT_DoT}, that our domain-level attention is helpful in reducing domain discrepancy. Notice that both the supervised structure learning strategy (with $\mathcal{L}_s$) and the unsupervised strategy (with $\mathcal{L}_t$ and $\mathcal{L}_{ent}$) can improve the basic model using $\mathcal{L}_{tce}$ only. We also see that DoT with all the loss terms consistently achieves the best performance. It suggests that the local structure regularization terms and domain-level attention mechanism can benefit from each other in learning transferable and discriminative features.

Compared with CNNs backbones, the ViT-B backbone shows promising performance. DoT based on ViT-B based improves the ViT-B baseline with higher mean accuracies on all the datasets, which shows that PA and DA play a complementary role in the learning process. To validate the necessity of DA fairly, we consider PA and DA at the same position (latent features), where PA employs single-head attention like that in ViT. PA focuses on establishing correspondence learning for each feature vector individually, while DA characterizes dependencies across feature vectors. The results are shown in the bottom row of Table~\ref{tab:ablation_experiment}. We can see that PA with the locality regularization terms performs well on ImageCLEF, but its performance is relatively weaker on Office-31 and VisDA-2017. This could be attributed to the challenge of capturing cross-domain interactions effectively with the PA module. Office-31 and VisDA-2017 involve substantial domain shifts, which makes it crucial to model the relations between cross-domain samples. In comparison, DoT consistently outperforms the new combination approach with a significant margin on more demanding tasks. These results indicate that domain-level attention exhibits a clear advantage over patch-level attention in dealing with UDA.

\begin{table}[t]
\caption{The performance of DoT based on PCA and LPP (\%).}
 \label{tab:ablation_DoT_PCA_LPP}
 \renewcommand{\tabcolsep}{0.28pc}
 \centering
 \begin{tabular}{c|cc|c|cc}
 \toprule
 \multirow{2}{*}{Principle}  & \multicolumn{2}{c|}{ImageCLEF} & VisDA-2017 &\multicolumn{2}{c}{Office-31}\\
& C$\rightarrow$I & I$\rightarrow$C & S$\rightarrow$R& A$\rightarrow$W & W$\rightarrow$A\\
 \midrule
 PCA & 92.3 & 96.7 & 82.0 & 93.3 & 82.5 \\
 LPP	& 95.3 & 97.8 &	84.9 & 94.8 & 83.1 \\
 \bottomrule
 \end{tabular}
\end{table}

\begin{figure*}[t]
    \centering
    \includegraphics[scale=0.9]{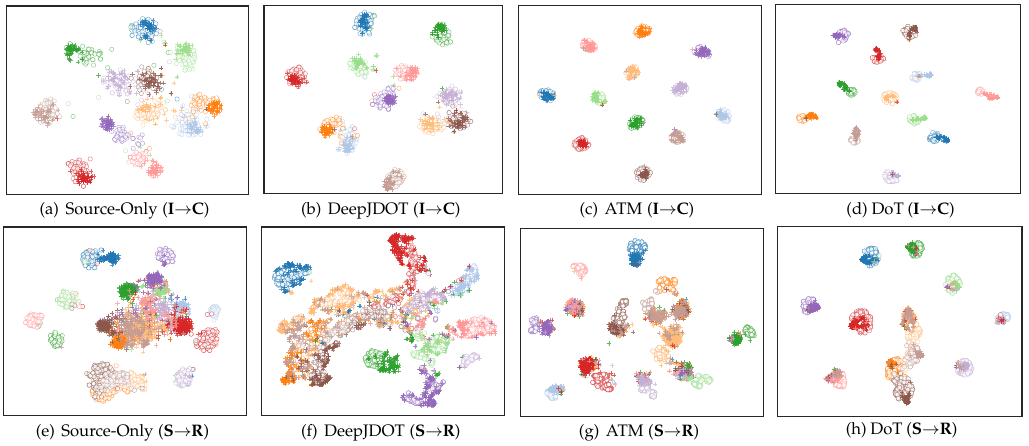}
    \caption{T-SNE visualization of features generated by Source-Only, DeepJDOT, ATM and DoT on ImageCLEF task I$\rightarrow$C and VisDA-2017 task S$\rightarrow$R, respectively. Here, ``o'' denotes source domain and ``+'' denotes target domain. Each color denotes one class. Best viewed in color.}
    \label{fig:tsne}
\end{figure*}

\subsubsection{Parameter Analysis}
We investigate the parameter sensitivity of $\lambda_{s}$ and $\lambda_{t}$ on ImageCELF to understand their impact on promoting locality consistency. Parameters $\lambda_{s}$ and $\lambda_{t}$ act on the source locality preserving loss $\mathcal{L}_s$ and target locality preserving loss $\mathcal{L}_t$, respectively. The parameters $\lambda_{t}$ and $\lambda_{s}$ are searched from $\{1e$-3, 1$e$-2, 1$e$-1, 1$e$0, 1$e$1, 1$e2\}$. Besides, parameter $\lambda_{ent}$ for the entropy loss $\mathcal{L}_{ent}$ is fixed as 0.001. As shown in Fig.~\ref{fig:hyper-param}, we can see that the accuracies will decrease with smaller values of $\lambda_{t}$, which demonstrates that exploring the intra-class compactness of the target domain can promote better adaptation performance. Besides, the accuracy decreases slowly among the peak area, which validates that DoT is stable under different parameter values of the local structure regularization.

\subsubsection{Principle Comparison}

In DoT, the local regularization terms are based on the LPP algorithm since it can satisfy the requirements of both supervised learning and unsupervised learning. Principle Component Analysis (PCA) learns the representations via maximizing the covariance, which characterizes the global structure and only serves as an unsupervised learning principle. To validate the superiority of LPP, we evaluate DoT by replacing the LPP-based regularization terms with PCA, i.e., maximizing the covariance of target features and covariance of the transformed source features. The results are shown in Table~\ref{tab:ablation_DoT_PCA_LPP}. We can observe that DoT based on LPP indeed ensures significantly higher accuracies on different datasets. Thus, LPP is a preferable principle for exploring local structures.

\subsection{Visualization and Analysis}
\subsubsection{Locality Consistency}
To explore the sample correspondence across domains, we visualize the attention matrices of different methods on ImageCLEF task \textbf{I}$\rightarrow$\textbf{C}. The attention matrix $\mathbf{A}$ defined in Eq.~\eqref{Eq:softmax-DoT} characterizes the pair-wise similarities between the source and target features. Since the attention matrix is theoretically connected with the barycenter map induced by the transport plan in Eq.~\eqref{eq:barycentric-map}, we regard the row-normalized optimal transport plan $\mbox{diag}(\bm{\gamma}^{\ast}\mathbf{1}_{n_t})^{-1}\bm{\gamma}^{\ast}$ as the attention matrix $\mathbf{A}$ for DeepJDOT. Oracle indicates the ideal case for alignment. The oracle matrix entry equals $1$ if the sample pairs within the same class; otherwise, it is $0$. Then, we normalize the non-zero elements of each row. Overall, all the visualization matrices in Fig.~\ref{fig:heatmap_imageclef-i2c} are row-normalized with the same scale. The block-diagonal color will be darker since samples are sorted by their classes. In Fig.~\ref{fig:heatmap_imageclef-i2c}(a), the correspondence between class-relevant samples is weak since Source-Only is a model without adaptation. Thus, it is natural that DeepJDOT and DoT have more obvious block-diagonal structures than Source-Only. Since DeepJDOT is based on the Kantorovich problem, the optimal transport plan $\bm{\gamma}^{\ast}$ is a sparse matrix. It is the reason that we can observe much darker diagonal dots in Fig.~\ref{fig:heatmap_imageclef-i2c}(b). Different from DeepJDOT, DoT captures class-relevant sample correspondence across domains via the domain-level attention and structure learning strategy. In Fig.~\ref{fig:heatmap_imageclef-i2c}(c), we can observe more obvious diagonal blocks instead of darker diagonal dots, which ensures that DoT can ensure locality consistency across domains. Though there is room for improvement compared with the Oracle case in Fig.~\ref{fig:heatmap_imageclef-i2c}(d), DoT has achieved the class-wise domain alignment without utilizing the soft labels or pseudo-labels of the uncertain target samples.

\begin{figure}[t]
    \centering
    \includegraphics[scale=0.36]{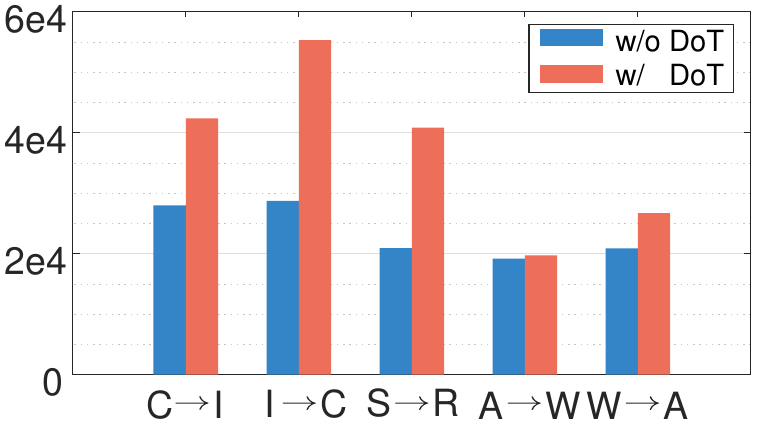}
    \caption{The F-norm of the cross-covariance matrix based on the embedded features.}
    \label{fig:f-norm}
\end{figure}

\subsubsection{Domain Correlation}
In statistics, the cross-covariance matrix can measure the correlation degree between two random variables. Mathematically, the cross-covariance matrix $\Sigma_{\mathbf{F}^s \mathbf{F}^t}$ is estimated by $\mathbb{E}[(\mathbf{F}^s -\mathbb{E}[\mathbf{F}^s])(\mathbf{F}^t -\mathbb{E}[\mathbf{F}^t])]$, and similarly for $\Sigma_{\hat{\mathbf{F}}^s \mathbf{F}^t}$. Here, we compute $\Sigma_{\mathbf{F}^s \mathbf{F}^t}$ based on the embedded features $\mathbf{F}^s$ and $\mathbf{F}^t$ from the Source-Only model, and obtain $\Sigma_{\hat{\mathbf{F}}^s \mathbf{F}^t}$ based on the features $\hat{\mathbf{F}}^s$ and $\mathbf{F}^t$ from DoT. We compute the Frobenius-norm (F-norm) of the cross-covariance matrices to compare the dependency between the learned domain-invariant features. Since F-norm characterizes the global correlation degree for the cross-covariance matrix, a larger norm indicates a larger correlation degree between the source and target domains.

The experiments are conducted on ImageCLEF, VisDA-2017, and Office-31. In Fig.~\ref{fig:f-norm}, we denote $\|\Sigma_{\mathbf{F}^s \mathbf{F}^t}\|_F$ as w/o DoT and $\|\Sigma_{\hat{\mathbf{F}}^s \mathbf{F}^t}\|_F$ as w/ DoT. We see that DoT has larger F-norms on all of these UDA tasks. It means that the sample correlations across different domains are improved. Considering the classification performance obtained by DoT, we think the larger norms are closely related to the model's generalization ability. This is consistent with the empirical results and important conclusions shown by \citep{xu2019larger}. The results validate that DoT achieves domain alignment and then improves the mode's generalization performance by capturing the sample correspondence between domains.

\begin{figure}[t]
    \centering
    \includegraphics[scale=0.38]{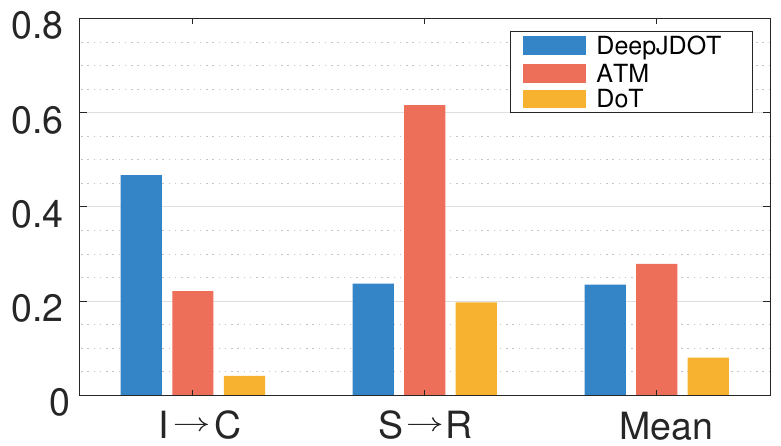}
    \caption{$\MC{W}_2$-distance between domains.}
    \label{fig:w2-distance}
\end{figure}

\begin{figure*}[t]
    \centering
    \includegraphics[scale=0.87]{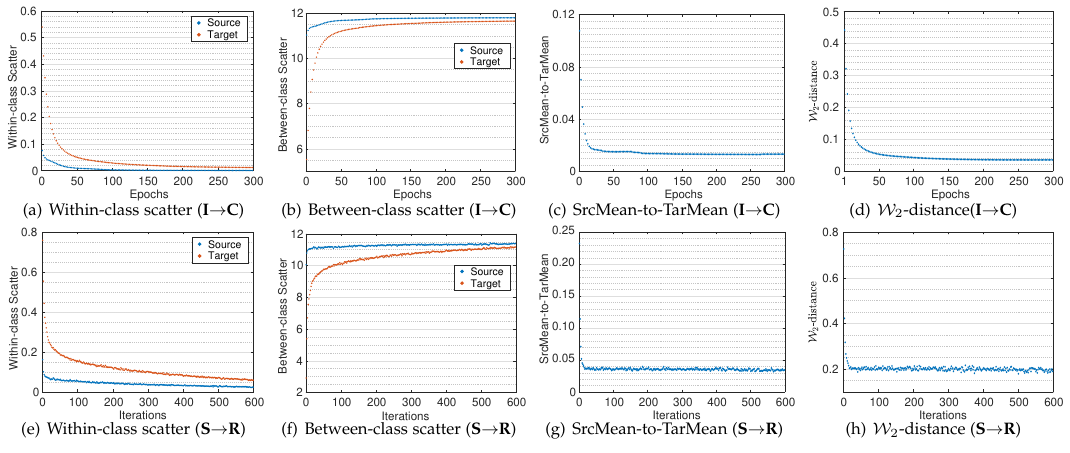}
    \caption{Scatter analysis for the embedded features on ImageCLEF task I$\rightarrow$C and VisDA-2017 task S$\rightarrow$R. $\hat{\mathbf{F}}^s$ is used for the source domain, and $\mathbf{F}^t$ for the target domain. Best viewed in color.}
   \label{fig:dist}
\end{figure*}

\subsubsection{Wasserstein Metric}
In the part of theoretical derivation, we establish the estimation error and generalization bound based on the $\MC{W}_2$-distance. To verify the rationality and effectiveness of these theoretical results, $\MC{W}_2$-distance is calculated for the features obtained by different methods and then presented in Fig.~\ref{fig:w2-distance}. Specifically, ImageCLEF task I$\rightarrow$C and VisDA-2017 task S$\rightarrow$ R are still used for simplicity. The features obtained by DeepJDOT, ATM, and DoT are input into the Sink-horn's iterative algorithm, respectively, and then the $\MC{W}_2$-distances are obtained. We can see that on both learning tasks, the $\MC{W}_2$-distance of DoT is smaller than those obtained by DeepJDOT and ATM, and the comparison on S$\rightarrow$R task is more obvious. ``Mean'' represents the average $\MC{W}_2$-distance corresponding to two groups of tasks, and DoT still keeps the superiority on distance value. Fig.~\ref{fig:w2-distance} shows that DoT effectively reduces the domain discrepancy without introducing the complex iteration computation. It also verifies the effectiveness of the learning theory and model based on the $\MC{W}_2$-distance again.

\subsubsection{Feature Visualization}
To intuitively evaluate the aligned features in lower-dimensional subspace, we use t-SNE~\citep{maaten2008visualizing} to visualize the features of Source-Only, DeepJDOT, ATM, and DoT on ImageCLEF task I$\rightarrow$C and VisDA-2017 task S$\rightarrow$R. The results are shown in Fig.~\ref{fig:tsne}. For VisDA-2017, we randomly selected 2000 samples across 12 categories from the source and target domains according to the original category ratio. Fig.~\ref{fig:tsne}(a)-(d) show that DeepJDOT, ATM, and DoT all improve the result of Source-Only by adaptation. Compared with DeepJDOT and ATM, DoT achieves better intra-class compactness since samples with the same class are closer. In Fig.~\ref{fig:tsne}(e), the source features have an obvious discriminative structure while the target features collide into a mess due to the large synthetic-to-real domain gap in VisDA-2017. As shown in Fig.~\ref{fig:tsne}(f), DeepJDOT achieves a pair-wise matching across domains by seeking an optimal transport plan $\bm{\gamma}$ in the joint space of features and labels. However, some aligned classes are not separable by a large margin. Fig.~\ref{fig:tsne}(g) further indicates that the MDD loss guarantees ATM to learn discriminative features. In Fig.~\ref{fig:tsne}(h), we observe that DoT can match the complex structures across domains along with achieving the intra-class compactness on the target domain.

\subsubsection{Scatter Analysis}
DoT learns transferable features by capturing sample correspondence, rather than explicitly modeling domain discrepancy. To verify the domain alignment brought by DoT, we measure four distances between the target and transformed source domains in the training process: within-class scatter, between-class scatter, SrcMean-to-TarMean, and $\mathcal{W}_2$-distance. Within-class scatter is the mean Euclidean distance between all samples and their corresponding class center. Between-class scatter computes the mean Euclidean distance between all class centers and the domain center, which is the average over samples from the target and transformed source domains. SrcMean-to-TarMean is the mean Euclidean distance between the source and target class centers. $\mathcal{W}_2$-distance is the 2-Wasserstein distance between the target and transformed source domains.

\begin{figure}[t]
    \centering
    \includegraphics[scale=0.95]{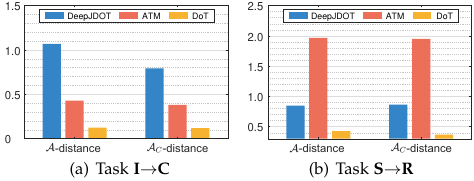}
    \caption{$\mathcal{A}$-distance and $\mathcal{A}_C$-distance between domains.}
    \label{fig:a-distance}
\end{figure}

Fig.~\ref{fig:dist} shows the four distances on ImageCLEF and VisDA-2017. In Fig.~\ref{fig:dist}(a), we can observe that the target domain has a much larger within-class scatter than the transformed source domain at first, and then reduces to 0, which suggests that DoT brings intra-class compactness. The between-class scatter of the target domain in Fig.~\ref{fig:dist}(b) is firstly much lower than that of the labeled transformed source domain, and then stabilizes at a high-level along with the training process, indicating that DoT achieves a large inter-class separability. The decreasing SrcMean-to-TarMean distance in Fig.~\ref{fig:dist}(c) verifies that DoT can enhance the class-wise domain alignment via locality consistency. The gradually decreasing $\mathcal{W}_2$-distance in Fig.~\ref{fig:dist}(d) demonstrates that DoT minimizes the generalization upper bound with the empirical Wasserstein distance. Similar curves are shown in Fig.~\ref{fig:dist}(e)-(h) on the S$\rightarrow$R task. Generally speaking, all these results indicate that DoT successfully learns locality consistency with the domain-level attention mechanism.

\subsubsection{Domain Discrepancy}
$\mathcal{A}$-distance~\citep{ben2010theory} has been widely used to measure domain discrepancy. $\mathcal{A}$-distance is estimated by $d_\mathcal{A} = 2(1-2\varepsilon)$, where $\varepsilon$ denotes the test error of a classifier trained for distinguishing the source and target domains. Thus, the smaller $\mathcal{A}$-distance, the better global domain alignment. To measure class-wise domain discrepancy, we estimate $\mathcal{A}_C$-distance by $d_{\mathcal{A}_C}\!=\!\mathbb{E}[d_{\mathcal{A}_c}]$, where $d_{\mathcal{A}_c}$ is the $\mathcal{A}$-distance for samples from the same class but different domains, and $\mathbb{E}[\cdot]$ is the expectation on the target domain. The smaller $\mathcal{A}_C$-distance, the better class-wise domain alignment. We provide the $\mathcal{A}$-distance and $\mathcal{A}_C$-distance of DeepJDOT, ATM and DoT on ImageCLEF and VisDA-2017, respectively. In Fig.~\ref{fig:a-distance}, ATM has smaller $\mathcal{A}$-distance and $\mathcal{A}_C$-distance than DeepJDOT on task I$\rightarrow$C while contrary performance on task S$\rightarrow$R. We also notice that DoT achieves the smallest $\mathcal{A}$-distance and $\mathcal{A}_C$-distance on both tasks, which validates that DoT is effective in reducing domain discrepancy. Besides, $\mathcal{A}_C$-distances are smaller than $\mathcal{A}$-distances. This can be attributed to the locality consistency learned by the domain-level attention mechanism in DoT.

\section{Conclusion}\label{sect:conclusion}

In this paper, we propose a novel method called Domain-transformer (DoT) for UDA. We define the domain-level attention inspired by the core attention mechanism of Transformer. With the structure learning strategy, DoT achieves discriminative feature extraction and transfer across domains. Theoretically, we connect the domain-level attention with OT theory and derive generalization error bounds with the Wasserstein geometry. Under the OT-like map, DoT ensures locality consistency across domains and avoids explicit divergence optimization. DoT can be applied to backbones based on CNNs or vision transformers, and it can be optimized in an end-to-end manner. In our future work, we plan to employ the merits of the transformers on more general problems, including meta-learning.

\bmhead{Acknowledgments}
This work is supported in part by National Natural Science Foundation of China (Grant No. 61976229, 62376291), in part by Guangdong Basic and Applied Basic Research Foundation (2023B1515020004), in part by the Open Research Projects of Zhejiang Lab (Grant 2021KH0AB08), in part by Guangdong Province Key Laboratory of Computational Science at Sun Yat-sen University (2020B1212060032), in part by Key Laboratory of Machine Intelligence and Advanced Computing, Ministry of Education, and in part by the Hong Kong Innovation and Technology Commission (InnoHK Project CIMDA).

\section*{Declarations}
\begin{itemize}
\item The datasets adopted can be requested and downloaded through the following links:
\begin{itemize}
\item[-]ImageCLEF~\href{https://www.imageclef.org/2014/adaptation}{link}
\item[-]Office-31~\href{https://faculty.cc.gatech.edu/~judy/domainadapt/#datasets_code}{link}
\item[-]Office-Home~\href{https://www.hemanthdv.org/officeHomeDataset.html}{link}
\item[-]VisDA-2017~\href{https://ai.bu.edu/visda-2017/}{link}
\item[-]DomainNet~\href{https://ai.bu.edu/M3SDA/#dataset}{link}
\end{itemize}
\item The code will be publicly open upon acceptance.
\end{itemize}

\bibliography{DoT_abbre}
\end{document}